\documentclass[11pt,letterpaper]{article}

\usepackage[letterpaper,total={6.5in,8.5in}]{geometry}
\usepackage[toc,page]{appendix}
\usepackage{xcolor}

\usepackage[nocompress]{cite}

\usepackage[]{graphicx}
\graphicspath{{./fig/}}
\usepackage{epstopdf}
\DeclareGraphicsExtensions{.pdf,.jpeg,.png,.eps}

\usepackage[cmex10]{amsmath}
\usepackage[hyphens]{url}
\usepackage{amsthm,amssymb,amsfonts,mathrsfs,bm,mathtools}
\usepackage{algorithm,algorithmicx,algpseudocode}
\usepackage{ragged2e}
\usepackage{tikz}
\usepackage{empheq}
\usepackage[inline]{enumitem}

\usepackage[caption=false,font=footnotesize,labelfont=sf,textfont=sf]{subfig}

\DeclareMathOperator*{\Argmin}{arg\;min}

\DeclareMathOperator{\prob}{Pr}
\DeclareMathOperator{\Expect}{\mathbb{E}}
\DeclareMathOperator{\arc}{arc}

\newlength\dlf  

\theoremstyle{definition}
\newtheorem{ass}{Assumption}
\newtheorem{remark}{Remark}
\newtheorem{theorem}{Theorem}
\newtheorem{definition}{Definition}

\title{Large-scale subspace clustering \\ using sketching and
	validation}
\author{Panagiotis~A.~Traganitis,
	Konstantinos~Slavakis,
	and~Georgios~B.~Giannakis,
	\thanks{Panagiotis A. Traganitis and Georgios B. Giannakis are
		with the Dept.\ of Electrical and Computer Engineering and the
		Digital Technology Center, University of Minnesota,
		Minneapolis, MN 55455, USA.\protect\\
		Konstantinos Slavakis is with the Dept.\ of Electrical Engineering,
		University at Buffalo, State University of New York, 
		Buffalo, NY 14260, USA. \protect\\
		This work was supported by the NSF grants 1442686, 1500713,
		Eager~1343860 and the AFOSR grant MURI-FA9550-10-1-0567.\protect\\ 
		E-mails: traga003@umn.edu, kslavaki@buffalo.edu, georgios@umn.edu}
}
\date{\vspace{-5ex}}

\begin{document}
\maketitle

  \begin{abstract}\justifying
    The nowadays massive amounts of generated and communicated data
    present major challenges in their processing. While capable
    of successfully classifying nonlinearly separable objects in
    various settings, subspace clustering (SC) methods incur
    prohibitively high computational complexity when processing
    large-scale data. Inspired by the random sampling and consensus
    (RANSAC) approach to robust regression, the present paper
    introduces a randomized scheme for SC, termed sketching and
    validation (SkeVa-)SC, tailored for large-scale data. At the heart
    of SkeVa-SC lies a randomized scheme for approximating the
    underlying probability density function of the observed data by
    kernel smoothing arguments. Sparsity in data representations is
    also exploited to reduce the computational burden of SC, while
    achieving high clustering accuracy. Performance analysis as well
    as extensive numerical tests on synthetic and real data
    corroborate the potential of SkeVa-SC and its competitive
    performance relative to state-of-the-art scalable SC approaches.
  \end{abstract}



\section{Introduction}
\label{sec:introduction}
\

The turn of the decade has trademarked society and
computing research with a ``data
deluge''~\cite{bigdata.economist.10}. As the number of smart and
internet-capable devices increases, so does the amount of data that is
generated and collected. While it is desirable to mine information
from this data, their sheer amount and dimensionality introduces
numerous challenges in their processing and pattern analysis, since available statistical
inference and machine learning approaches do not necessarily scale
well with the number of data and their dimensionality. In addition, as
the cost of cloud computing is rapidly declining~\cite{Bezoslaw},
there is a need for redesigning those traditional approaches to take
advantage of the flexibility that has emerged from distributing
required computations to multiple nodes, as well as reducing the
per-node computational burden.

Clustering (unsupervised classification) is a method of grouping
unlabeled data with widespread applications across the fields of data
mining, signal processing, and machine learning. $K$-means is one of
the most successful clustering algorithms due to its
simplicity~\cite{bishop2006pattern}. However, $K$-means, as well as
its kernel-based variants, provides meaningful clustering results only
when data, after mapped to an appropriate feature space, form
``tight'' groups that can be separated by
hyperplanes~\cite{hastie01statisticallearning}.

Subspace clustering (SC) is a popular method that can group
non-linearly separable data which are generated by a union of (affine)
subspaces in a high-dimensional Euclidean
space~\cite{vidal2010tutorial}. SC has well-documented impact in various applications,
as diverse as image and video segmentation and identification of
switching linear systems in controls~\cite{vidal2010tutorial}. Recent
advances advocate SC algorithms with high clustering performance at
the price of high computational complexity~\cite{vidal2010tutorial}.

The goal of this paper is to introduce a randomized framework for
reducing the computational burden of SC algorithms when the number of
available data becomes prohibitively large, while maintaining high
levels of clustering accuracy. The starting point
is our sketching and validation (SkeVa) approach in~\cite{skeva}. SkeVa
offers a low-computational complexity, randomized scheme for
multimodal probability density function (pdf) estimation, since it
draws random computationally affordable sub-populations of data to obtain a crude
\textit{sketch} of the underlying pdf of the massive data
population. A \textit{validation} step, based on divergences of pdfs,
follows to assess the quality of the crude sketch. These sketching
and validation phases are repeated independently for a pre-fixed
number of times, and the draw achieving the ``best score'' is finally
utilized to cluster the whole data population. SkeVa is inspired from
the random sampling and consensus (RANSAC)
method~\cite{RANSAC81}. Although RANSAC's principal field of
application is the parametric regression problem in the presence of
outliers, it has been also employed in SC~\cite{vidal2010tutorial}.

To achieve the goal of devising an SC scheme with low computational
complexity footprint, the present paper broadens the scope of SkeVa to
the SC context. Moreover, to support the state-of-the-art performance
of SkeVa on real-data, this contribution provides a rigorous performance
analysis on the lower bound of the number of independent draws for
SkeVa to identify a draw that ``represents well'' the underlying data
pdf, with high probability. The analysis is facilitated by the
non-parametric estimation framework of kernel
smoothing~\cite{wand1994kernel}, and models the underlying
data pdf as a mixture of Gaussians, which are known to offer universal pdf
approximations~\cite{universalapprox, gpml}. To assess the proposed
SkeVa-SC, extensive numerical tests on synthetic and real-data are
presented to underline the competitive performance of SkeVa-SC
relative to state-of-the-art scalable SC approaches.

The rest of the paper is organized as
follows. Section~\ref{sec:prelim} provides SC preliminaries along
with notation and tools for kernel density estimation.
Section~\ref{sec:proposed_algorithm} introduces the proposed algorithm
for large-scale SC, while Section~\ref{sec:performance}
provides performance bounds for
SkeVa-SC. Section~\ref{sec:numerical_tests} presents numerical tests
conducted to evaluate the performance of SkeVa-SC in comparison with
state-of-the-art SC algorithms. Finally, concluding remarks and future
research directions are given in Section~\ref{sec:conclusion}.

\noindent\textbf{Notation:} Unless otherwise noted, lowercase bold letters, $\bm{x}$, denote
vectors, uppercase bold letters, $\mathbf{X}$, represent matrices, and
calligraphic uppercase letters, $\mathcal{X}$, stand for sets. The
$(i,j)$th entry of matrix $\mathbf{X}$ is denoted by
$[\mathbf{X}]_{ij}$; $\mathbb{R}^{D}$ stands for the
$D$-dimensional real Euclidean space, $\mathbb{R}_{+}$ for the set of positive real numbers, $\Expect[\cdot]$ for
expectation, and $\|\cdot\|$ for the $L_2$-norm.

\section{Preliminaries}\label{sec:prelim}

\subsection{SC problem statement}\label{ssec:subspaceclustering}

Consider data $\{\bm{x}_{i}\in\mathbb{R}^{D}\}_{i=1}^{N}$ drawn from a
union of $K$ affine subspaces, each denoted by
$\mathcal{S}_k$, adhering to the model:
\begin{equation}
  \label{eq:pointsubspace}
  \bm{x}_i = \mathbf{U}_k\bm{y}_i + \bm{m}_k + \bm{v}_i\,,\quad
  \forall\bm{x}_i\in\mathcal{S}_k 
\end{equation}
where $d_k$ (possibly with $d_k\ll D$) is the dimensionality of
$\mathcal{S}_k$; $\mathbf{U}_k$ is a $D\times d_k$ matrix whose
columns form a basis of $\mathcal{S}_k$, $\bm{y}_i\in\mathbb{R}^{d_k}$
is the low-dimensional representation of $\bm{x}_i$ in
$\mathcal{S}_k$ with respect to (w.r.t.) $\mathbf{U}_k$,
$\bm{m}_k\in\mathbb{R}^{D}$ is the ``centroid'' or intercept of
$\mathcal{S}_k$, and $\bm{v}_i\in\mathbb{R}^{D}$ denotes the noise
vector capturing unmodeled effects. If $\mathcal{S}_k$ is linear then
$\bm{m}_k = \bm{0}$. Using \eqref{eq:pointsubspace}, \textit{any} $\bm{x}_i$ can be described as 
\begin{equation}
  \label{eq:allpoints}
  \bm{x}_i = \sum_{k=1}^{K}[\bm{\pi}_{i}]_k\left(\mathbf{U}_k\bm{y}_i
    + \bm{m}_k\right) + \bm{v}_i
\end{equation}
where $\bm{\pi}_i$ is the cluster assignment vector for 
$\bm{x}_i$ and $[\bm{\pi}_{i}]_k$ denotes the $k$th entry of
$\bm{\pi}_i$ under the constraints of $[\bm{\pi}_{i}]_k\ge 0$ and
$\sum_{k=1}^{K}[\bm{\pi}_{i}]_k = 1$. If $\bm{\pi}_i\in\{0,1\}^{K}$
then datum $\bm{x}_i$ lies in only one subspace (hard
clustering), while if $\bm{\pi}_i\in[0,1]^{K}$, then $\bm{x}_i$
can belong to multiple clusters (soft clustering). In the latter case,
$[\bm{\pi}_{i}]_k$ can be thought of as the probability that datum
$\bm{x}_i$ belongs to $\mathcal{S}_k$.

Given the data matrix $\mathbf{X} :=
[\bm{x}_1,\bm{x}_2,\ldots,\bm{x}_N]\in\mathbb{R}^{D\times N}$ and the
number of subspaces $K$, SC involves finding the data-to-subspace
assignment vectors $\{\bm{\pi}_i\}_{i=1}^{N}$, the subspace bases
$\left\{\mathbf{U}_k\right\}_{k=1}^{K}$, their dimensions
$\{d_k\}_{k=1}^{K}$, the low-dimensional representations
$\{\bm{y}_i\}_{i=1}^{N}$, as well as the centroids 
$\{\bm{m}_k\}_{k=1}^{K}$~\cite{vidal2010tutorial}. SC can
be formulated as follows 
\begin{equation} 
  \label{eq:subspaceclustering}
  \begin{aligned}
    &\underset{\mathbf{\Pi},\{\mathbf{U}_k\},\{\bm{y}_i\},\mathbf{M}}{\min} 
    &&\sum_{k=1}^{K}\sum_{i=1}^{N} [\bm{\pi}_{i}]_k\|\bm{x}_i -
    \mathbf{U}_k\bm{y}_i - \bm{m}_k \|_2^2 \\ 
    &\text{subject to (s.to)} && \mathbf{\Pi}^{\top}\bm{1} = \bm{1};
    \quad [\bm{\pi}_{i}]_k\geq 0,\ \forall (i,k)
  \end{aligned}
\end{equation}
where $\mathbf{\Pi} := [\bm{\pi}_1,\ldots,\bm{\pi}_N]$, $\mathbf{M} :=
[\bm{m}_1,\bm{m}_2,\ldots,\bm{m}_K]$,
and $\bm{1}$ denotes the all-ones vector of matching dimensions.

The problem in \eqref{eq:subspaceclustering} is non-convex as all of
$\mathbf{\Pi}, \{\mathbf{U}_k\}_{k=1}^{K}, \{d_k\}_{k=1}^{K}, \{\bm{y}_i\}_{i=1}^{N}$,
and $\mathbf{M}$ are unknown. We outline next a popular
alternating way of solving \eqref{eq:subspaceclustering}. For given
$\bm{\Pi}$ and $\{d_k\}_{k=1}^{K}$, bases of the subspaces can be
recovered using the singular value decomposition (SVD) on the data associated with each subspace. Indeed, given
$\mathbf{X}_k := [\bm{x}_{i_1}, \ldots, \bm{x}_{i_{N_k}}]$, associated with $\mathcal{S}_k$
($\sum_{k=1}^K N_k = N$), a basis $\mathbf{U}_k$ can be obtained from the first $d_k$ (from
the left) singular vectors of
$\mathbf{X}_k - [\bm{m}_k, \ldots, \bm{m}_k]$ where $\bm{m}_k =
(1/N_k)\sum_{i\in\mathcal{S}_k}\bm{x}_i$. On the other hand, when
$\{\mathbf{U}_k, \bm{m}_k\}_{k=1}^K$ are given, the assignment
matrix $\bm{\Pi}$ can be recovered in the case of hard clustering by
finding the closest subspace to each datapoint; that is, $\forall i\in
\{1,2,\ldots,N\}$, $\forall k\in\{1, \ldots, K\}$, we obtain
\begin{equation}
  \label{eq:findclosestsubspace}
  [\bm{\pi}_{i}]_k = \begin{cases}
    1, & \text{if}\ k = \Argmin\limits_{k'\in \{1, \ldots, K\}}
    \left\| \bm{x}_i - \bm{m}_{k'} -
      \mathbf{U}_{k'}\mathbf{U}_{k'}^{\top} \bm{x}_i\right\|_2^2\\   
    0, & \text{otherwise}
  \end{cases}  
\end{equation}
where $\| \bm{x}_i - \bm{m}_k - \mathbf{U}_{k}\mathbf{U}_{k}^{\top} \bm{x}_i
\|_2$ is the distance of $\bm{x}_i$ from $\mathcal{S}_k$.

The $K$-subspaces algorithm~\cite{ksubspaces}, which is a
generalization of the ubiquitous $K$-means one~\cite{lloydkmeans} for
SC, builds upon this alternating minimization of
\eqref{eq:subspaceclustering}: (i) First, fixing $\mathbf{\Pi}$ and
then solving for the remaining unknowns; and (ii) fixing
$\{\mathbf{U}_k, \bm{m}_k\}_{k=1}^K$, and then solving for
$\mathbf{\Pi}$. Since SVD is involved, SC entails high
computational complexity, whenever $d_k$ and/or $N_k$ are
massive. Specifically the SVD of a $D\times N_k$ matrix incurs a
computational complexity of $\mathcal{O}(\alpha D^2 N_k + \beta
N_k^3)$, where $\alpha, \beta$ are algorithm-dependent constants. 
  
It is known that when $K=1$ and
$\mathbf{U}$ is orthonormal, \eqref{eq:subspaceclustering} boils down
to PCA~\cite{jolliffe2002principal}
\begin{equation} 
  \label{eq:PCAclustering}
  \begin{aligned}
    & \underset{\mathbf{U},\{\bm{y}_i\},\mathbf{M}}{\min}
    && \sum_{i=1}^{N} \|\bm{x}_i -
    \mathbf{U}\bm{y}_i - \bm{m} \|_2^2 \\ 
    &\text{s.to} && \mathbf{U}^{\top} \mathbf{U} = \mathbf{I}_{d}
  \end{aligned}
\end{equation}
where $\mathbf{I}_{d}$ ($d := d_k$) stands for the $d\times d$
identity matrix. Notice that for $K=1$, it holds that
$[\bm{\pi}_{i}]_k=1$. Moreover, if
$\mathbf{U}_k := \bm{0}$, $\forall k$, with $K>1$ looking for
$\{\bm{m}_k,\bm{\pi}_k \}_{k=1}^{K}$ amounts to ordinary
clustering
\begin{equation} 
  \label{eq:normalclustering}
  \begin{aligned}
    &\underset{\mathbf{\Pi},\mathbf{M}}{\min}
    &&\sum_{k=1}^{K}\sum_{i=1}^{N} [\bm{\pi}_{i}]_k \|\bm{x}_i -
    \bm{m}_k \|_2^2 \\ 
    &\text{s.to} && \mathbf{\Pi}^{\top}\bm{1} = \bm{1} \,.
  \end{aligned}
\end{equation}
  
Finally, it is also well known that \eqref{eq:normalclustering}
with $\bm{\pi}\in[0,1]^K$ (soft clustering) amounts to pdf estimation
\cite{bishop2006pattern}. This fact will be exploited by our novel approach to perform on large-scale SC, and link the benefits of
pdf estimation with those of high-performance SC algorithms.

\subsection{Kernel density estimation}\label{ssec:kernel}

Kernel smoothing or kernel density estimation~\cite{wand1994kernel} is
a non-parametric approach to estimating pdfs. Kernel density
estimators, similar to the histogram and unlike parametric estimators,
make minimal assumptions about the unknown pdf, $f$. Because they
employ general kernel functions rather than rectangular bins, kernel
smoothers are flexible to attain faster convergence rates than the histogram estimator
as the number of observed data, $n$, tends to
infinity~\cite{wand1994kernel}. However, their convergence rate is
slower than that of parametric estimators. For data
$\{\bm{x}_i\}_{i=1}^{n} \sim f$, drawn from $f$, the kernel density
estimator of $f$ is given by~\cite{wand1994kernel}
\begin{equation}
  \hat{f}(\bm{x}) = \frac{1}{n}\sum_{i=1}^{n} K_{\mathbf{H}}
  (\bm{x}-\bm{x}_i)\
\end{equation}
where $K_{\mathbf{H}}(\bm{x}-\bm{x}_i)$ denotes a pre-defined kernel
function centered at $\bm{x}_i$, with a positive-definite bandwidth
matrix $\mathbf{H}\in\mathbb{R}^{D\times D}$. This bandwidth matrix
controls the amount of smoothing across dimensions. Typically, $K_{\mathbf{H}}(\bm{x})$ is chosen to be a
density so that $\hat{f}(\bm{x})$ is also a pdf. The
role of bandwidth can be understood clearly as $\mathbf{H}$
corresponds to the covariance matrix of the pdf $K_{\mathbf{H}}(\bm{x})$~\cite{wand1994kernel}.

The performance of kernel-based estimators is usually assessed using the mean-integrated square error (MISE)
\begin{equation}
  \label{eq:MISE}
  \text{MISE}(\hat{f};\mathbf{H}) := \Expect\left[\int\left(f(\bm{x}) -
      \hat{f}(\bm{x})\right)^2d\bm{x}\right]
\end{equation}
where $\int$ denotes $D$-dimensional integration, and $d\bm{x}$ stands for
$dx_1\ldots dx_D$, with $\bm{x} := [x_1, \ldots,x_D]^{\top}$. The
choice of $\mathbf{H}$ has a notable effect on MISE. As a result, it makes
sense to select $\mathbf{H}$ such that the MISE is low.

Letting $\mathbf{H} := h^2\mathbf{I}_D$ with $h>0$, facilitates the minimization of \eqref{eq:MISE} w.r.t.\ $h$, especially if $h$ is regarded as a
function of $n$, i.e., $h:= h(n)$, and satisfies
\begin{equation}
\label{eq:hrequirements}
\lim_{n\rightarrow\infty}h = 0, \quad \lim_{n\rightarrow\infty}nh = \infty\,.
\end{equation}
Using the $D$-variate Taylor expansion of $f$ the asymptotically MISE optimal $h$ is~\cite{wand1994kernel}:
\begin{equation}
  \label{eq:amiseoptimalbw}
  h^{*} = \left[\frac{D\int
      K^2(\bm{x})d\bm{x}}{n\int{x_i^2} K(\bm{x})
      d\bm{x} \int(\nabla^2 f(\bm{x}))^2d\bm{x}}\right]^{1/(D+4)}
\end{equation}
where $\int x_i^2K(\bm{x})d\bm{x}$ is not a function of $i$, when
$K(\bm{x})$ is a spherically symmetric compactly supported density,
and $\nabla^2 f(\bm{x}) := \sum_{i=1}^{D}\frac{\partial^2
  f(\bm{x})}{\partial x_i^2}$.  In addition, when the density $f$ to
be estimated is a mixture of Gaussians~\cite{universalapprox, gpml},
and the kernel is chosen to be a Gaussian density
$K_{\mathbf{H}}(\bm{x} - \bm{x}_i) := \phi_{\mathbf{H}}(\bm{x} -
\bm{x}_i)$, where
\begin{equation}
  \label{eq:multivariateGaussian}
  \begin{aligned}
     \phi_{\mathbf{H}}(\bm{x}-\bm{x}_i)
     := \frac{1}{(2\pi)^{D/2}|\mathbf{H}|^{1/2}}\exp \left(-\frac{1}{2}(\bm{x}-
    \bm{x}_i)^{\top}\mathbf{H}^{-1}(\bm{x}-\bm{x}_i)\right)
  \end{aligned}
\end{equation} 
with mean $\bm{x}_i$ and covariance matrix $\mathbf{H}$, it is
possible to express \eqref{eq:MISE} in closed form. This closed-form
is due to the fact that convolution of Gaussians is also Gaussian; that is, 
\begin{equation}
  \label{eq:gaussiantrickmult}
  \int\phi_{\mathbf{H}_i}(\bm{x}-\bm{x}_i)\phi_{\mathbf{H}_j}
  (\bm{x}-\bm{x}_j)d\bm{x} = \phi_{(\mathbf{H}_i +
    \mathbf{H}_j)} (\bm{x}_i - \bm{x}_j)\,.
\end{equation}
If $K_{\mathbf{H}}(\bm{x}) =
\phi_{\mathbf{I}_D}(\bm{x})$, the optimal bandwidth of
\eqref{eq:amiseoptimalbw} becomes
\begin{equation}
  \label{eq:optASbw}
  \begin{aligned}
    h^* = \left[\frac{D\phi_{2\mathbf{I}_D}(\bm{0})}{n
        \int(\nabla^2f(\bm{x}))^2d \bm{x}}\right]^{1/(D+4)}\,.
  \end{aligned}
\end{equation}
The proposed SC algorithm of Section~\ref{sec:proposed_algorithm} will draw
ideas from kernel density estimation, and a proper choice of $h$ will be
proved to be instrumental in optimizing SC performance.

\subsection{Prior work}\label{ssec:prior}

Besides the $K$-subspaces solver outlined in Sec.~\ref{ssec:subspaceclustering}, various algorithms have been
developed by the machine learning~\cite{vidal2010tutorial} and
data-mining community~\cite{parsons2004subspace} to solve
\eqref{eq:subspaceclustering}. A probabilistic (soft) counterpart of
$K$-subspaces is the mixture of probabilistic
PCA~\cite{tipping1999mixtures}, which assumes that data are drawn from
a mixture of degenerate (having zero variance in some dimensions) Gaussians. Building on the same
assumption, the agglomerative lossy compression (ALC)~\cite{ALC2007}
utilizes ideas from rate-distortion theory~\cite{ratedistorionberger}
to minimize the required number of bits to ``encode'' each cluster, up to a certain distortion level. Algebraic schemes, such as
the Costeira-Kanade algorithm~\cite{costeira} and Generalized PCA
(GPCA~\cite{gpca}), aim to tackle SC from a linear
algebra point of view, but generally their performance is guaranteed
only for independent and noise-less subspaces. Other
methods recover
subspaces by finding local linear subspace approximations~\cite{zhang2012hybrid}. See also \cite{zhang2009median,vidalonline} for online
clustering approaches to handling streaming data.

Arguably the most successful class of solvers for
\eqref{eq:subspaceclustering} relies on spectral
clustering~\cite{spectralclustering} to find the data-to-subspace
assignments. Algorithms in this class generate first an $N\times N$ affinity
matrix $\mathbf{A}$ to capture the
similarity between data, and then perform spectral clustering on
$\mathbf{A}$. Matrix $\mathbf{A}$ implies a graph $\mathcal{G}$ whose
vertices correspond to data and edge weights between
data are given by its entries. Spectral
clustering algorithms form the graph Laplacian matrix
\begin{equation}
  \label{eq:laplacian}
  \mathbf{L} := \mathbf{D} - \mathbf{A}
\end{equation}
where $\mathbf{D}$ is a diagonal matrix such
that (s.t.) $[\mathbf{D}]_{ii} = \sum_{j=1}^{N}
[\mathbf{A}]_{ij}$. The algebraic multiplicity of the $0$ eigenvalue
of $\mathbf{L}$ yields the number of connected components in
$\mathcal{G}$~\cite{spectralclustering}, while the corresponding eigenvectors are indicator vectors of the connected
components. For the sake of completeness, Alg.~\ref{alg:Spectral}
summarizes the procedure of how the trailing eigenvectors of
$\mathbf{L}$ are used to obtain cluster assignments in spectral
clustering.


\begin{algorithm}[t]
  \begin{algorithmic}[1]
    \algrenewcommand\algorithmicindent{1em}

    \Require{Data affinity matrix $\mathbf{A}$; number of clusters
      $K$.}
    
    \Ensure{Data-cluster associations.} 

    \State\parbox[t]{\dimexpr\linewidth-3\dimexpr\algorithmicindent}{Form
      diagonal matrix $\mathbf{D}$, with entries $[\mathbf{D}]_{ii}
      := \sum_{j=1}^{N} [\mathbf{A}]_{ij}$.} 

    \State\parbox[t]{\dimexpr\linewidth-3\dimexpr\algorithmicindent}{Laplacian
      matrix: $\mathbf{L} := \mathbf{D} - \mathbf{A}$.} 

    \State\parbox[t]{\dimexpr\linewidth-3\dimexpr\algorithmicindent}{Extract
      $K$ trailing eigenvectors
      $\{\mathbf{v}_k\in\mathbb{R}^{N}\}_{k=1}^{K}$ of
      $\mathbf{L}$. Let $\mathbf{V} =
      [\mathbf{v}_1,\ldots,\mathbf{v}_K]\in\mathbb{R}^{N\times K}$.} 
    
    \State\parbox[t]{\dimexpr\linewidth-3\dimexpr\algorithmicindent}{Let
      $\{\mathbf{z}_i\}_{i=1}^{N} $ be the rows of
      $\mathbf{V}$; $\mathbf{z}_i$ corresponds to the $i$th vertex
      ($i$th datapoint).} 
    
    \State\parbox[t]{\dimexpr\linewidth-3\dimexpr\algorithmicindent}{Group
      $\{\mathbf{z}_i\}_{i=1}^{N}$ into $k$ clusters using $K$-means.} 

  \end{algorithmic}
  \caption{Unnormalized spectral
    clustering~\cite{spectralclustering}} \label{alg:Spectral}  
\end{algorithm}

Sparse subspace clustering (SSC)~\cite{elhamifar2013SSC} exploits the
fact that under the union of subspaces model, \eqref{eq:subspaceclustering}, only a small percentage
of data suffice to provide a low-dimensional affine representation of
any $\bm{x}_i$, i.e., $\bm{x}_i = \sum_{j=1,j\neq
  i}^{N}w_{ij}\bm{x}_j$, $\forall i\in\{1,2,\ldots,N\}$. Specifically,
SSC solves the following sparsity-imposing optimization problem
\begin{equation}
  \label{eq:SSC}
  \begin{aligned}
    & \underset{\mathbf{W}}{\min} && \|\mathbf{W}\|_1 +
    \lambda\|\mathbf{X} - \mathbf{X}\mathbf{W}\|_2^2 \\ 
    & \text{s.to} && \mathbf{W}^{\top}\bm{1} = \bm{1};\quad
    \text{diag}(\mathbf{W}) = \bm{0}
  \end{aligned}
\end{equation}
where $\mathbf{W} := [\bm{w}_1,\bm{w}_2,\ldots,\bm{w}_N]$; column
$\bm{w}_i$ is sparse and contains the coefficients for the
representation of $\bm{x}_i$; $\lambda>0$ is the regularization
coefficient; and $\|\mathbf{W}\|_1 :=
\sum_{i,j=1}^{N}[\mathbf{W}]_{i,j}$. Matrix $\mathbf{W}$ is used to
create the affinity matrix $[\mathbf{A}]_{ij} := |[\mathbf{W}]_{ij}| +
|[\mathbf{W}]_{ji}|$. Finally, spectral clustering, e.g.,
Alg.~\ref{alg:Spectral}, is performed on $\mathbf{A}$ and cluster
assignments are identified. Using those assignments, $\mathbf{M}$ is
found by taking sample means per cluster, and
$\{\mathbf{U}_k\}_{k=1}^{K}$, $\{\bm{y}_i\}_{i=1}^{N}$ are obtained by
applying SVD on $\mathbf{X}_k - [\bm{m}_k, \ldots, \bm{m}_k]$. For
future use, SSC is summarized in Alg.~\ref{alg:SSC}.

The low-rank representation algorithm (LRR)~\cite{LRR} is similar to
SSC, but replaces the $\ell_1$-norm in \eqref{eq:SSC} with the nuclear
one: $\|\mathbf{W}\|_* := \sum_{i=1}^{\rho}\sigma_i(\mathbf{W})$,
where $\rho$ stands for the rank and $\sigma_i(\mathbf{W})$ for the
$i$th singular value of $\mathbf{W}$. The high clustering accuracy
achieved by both SSC and LRR comes at the price of high
complexity. Solving \eqref{eq:SSC} scales quadratically with the
number of data $N$, on top of performing spectral clustering across
$K$ clusters, which renders SSC computationally prohibitive for
large-scale SC. When data are high-dimensional ($D\gg$), methods based
on (statistical) leverage scores, random projections~\cite{Drineas},
or our recent sketching and validation (SkeVa)~\cite{skeva}
approach can be employed to reduce complexity to an affordable level. When the number of data is
large ($N\gg$), the current state-of-the-art approach, scalable sparse
subspace clustering (SSSC)~\cite{SSSC}, involves drawing randomly
$n<N$ data, performing SSC on them, and expressing the rest of the
data according to the clusters identified by that random draw of
samples. While SSSC clearly reduces complexity, performance can
potentially suffer as the random sample may not be representative of
the entire dataset, especially when $n\ll N$ and clusters are
unequally populated. To alleviate this issue, the present paper
introduces a structured trial-and-error approach to identify a
``representative'' $n$-point sample from a dataset with $n\ll N$,
while maintaining low computational complexity.

\begin{algorithm}[h]
  \begin{algorithmic}[1]
    \algrenewcommand\algorithmicindent{1em}

    \Require{Data $\mathbf{X}$; number of clusters $K$; $\lambda$}
    \Ensure{Data-cluster associations $\mathbf{\Pi}$.} 

    \State\parbox[t]{\dimexpr\linewidth-3\dimexpr\algorithmicindent}{Solve
      \eqref{eq:SSC} for $\mathbf{W}$.}

    \State\parbox[t]{\dimexpr\linewidth-3\dimexpr\algorithmicindent}{$[\mathbf{A}]_{ij} 
      := |[\mathbf{W}]_{ij}| + |[\mathbf{W}]_{ji}|$.} 

    \State\parbox[t]{\dimexpr\linewidth-3\dimexpr\algorithmicindent}{Perform
      spectral clustering, namely, Alg.~\ref{alg:Spectral}, on $\mathbf{A}$.} 
    
    \State\parbox[t]{\dimexpr\linewidth-3\dimexpr\algorithmicindent}{Identify 
      point-to-subspace associations.} 
  \end{algorithmic}
  \caption{Sparse Subspace Clustering (SSC)
    \cite{elhamifar2013SSC}}\label{alg:SSC} 
\end{algorithm}

Regarding kernel smoothing, most of the available algorithms address
the important issue of bandwidth selection $\mathbf{H}$ to achieve
desirable convergence rate properties in the approximation of the
unknown pdf~\cite{hall1983large, scott1987biased,
  sheather1991reliable, hall1992smoothed}. The present paper, however,
pioneers a framework to randomly choose the subset of kernel functions
yielding a small error when estimating a multi-modal pdf.

\section{The SkeVa-SC algorithm}
\label{sec:proposed_algorithm}

The proposed algorithm, named SkeVa-SC,
is listed in Alg.~\ref{alg}. It aims at iteratively finding a
representative subset of $n$ randomly drawn data,
$\check{\mathbf{X}}\in\mathbb{R}^{D\times n}$, run subspace clustering
on $\check{\mathbf{X}}$, and associate the remaining data
$\tilde{\mathbf{X}} :=
\mathbf{X}\setminus\check{\mathbf{X}}\in\mathbb{R}^{D\times (N - n)}$
with the subspaces extracted from $\check{\mathbf{X}}$. SkeVa-SC draws a prescribed number of $R_{\max}$ realizations, over which two phases are performed: A sketching and a validation
phase. The structure and philosophy of SkeVa-SC is based on the notion of divergences between pdfs. 

Per realization $r$, {\it the sketching stage} proceeds as follows: A sub-population,
$\check{\mathbf{X}}^{(r)}$, of the data is randomly drawn, and a pdf
$\hat{f}^{(r)}(\bm{x})$, is estimated based on this sample using kernel smoothing. As the clusters are assumed to be sufficiently separated,
$\hat{f}^{(r)}(\bm{x})$ is expected to be multimodal. To confirm this,
$\hat{f}^{(r)}(\bm{x})$ is compared with a unimodal pdf
$\hat{f}_0^{(r)}(\bm{x})$, using a measure
$d(\hat{f}^{(r)},\hat{f}_0^{(r)})$ of pdf discrepancy. If
$\hat{f}^{(r)}$ is sufficiently different from $\hat{f}_0^{(r)}$ then $d(\hat{f}^{(r)},\hat{f}_0^{(r)}) \ge \Delta_0$, where
$\Delta_0$ is some pre-selected threshold, SkeVa-SC proceeds to the
validation stage; otherwise, SkeVa-SC deems this draw uninformative and advances to the next realization
$r+1$, without performing the validation step.

At the {\it validation stage} of SkeVa-SC, another random sample of $n'$
data, $(n\leq n'\ll N)$, different from the one in
$\check{\mathbf{X}}^{(r)}$, is drawn, forming
$\tilde{\mathbf{X}}^{(r)}\in\mathbb{R}^{D\times n'}$. The purpose of
this stage is to evaluate how well $\check{\mathbf{X}}^{(r)}$ represents the
whole dataset. The pdf $\tilde{f}^{(r)}(\bm{x})$ of
$\tilde{\mathbf{X}}^{(r)}$ is estimated and compared to
$\hat{f}^{(r)}(\bm{x})$ using $d(\hat{f}^{(r)},\tilde{f}^{(r)})$. A
score $\psi[d(\hat{f}^{(r)},\tilde{f}^{(r)})]$ is assigned to
$\check{\mathbf{X}}^{(r)}$, using a non-increasing scoring function
$\psi:\mathbb{R}\rightarrow\mathbb{R}$ that grows as realizations $\hat{f}^{(r)}$ and $\tilde{f}^{(r)}$ come closer.

Finally, after $R_{\max}$ realizations, the set of $n$ data
$\check{\mathbf{X}}^{(r^*)}$ that received the highest score $r^* :=
\arg\max_r \psi[d(\hat{f}^{(r)},\tilde{f}^{(r)})]$, is selected and SC
(SSC or any other algorithm) is performed on it; that is,
\begin{equation}
  \label{eq:SSConsampled}
  \begin{aligned}
    & \underset{\mathbf{W}}{\min} && \|\mathbf{W}\|_1 +
    \lambda\left\|\check{\mathbf{X}}^{(r^*)} -
    \check{\mathbf{X}}^{(r^*)}\mathbf{W}\right\|_2^2 \\ 
    & \text{s.to} && \mathbf{W}^{\top} \bm{1} = \bm{1};\quad
    \text{diag}(\mathbf{W}) = \bm{0}\,.
  \end{aligned}
\end{equation}
The remaining data $\tilde{\mathbf{X}}^{(r^*)} :=
\mathbf{X}\setminus\check{\mathbf{X}}^{(r^*)}$ are associated with the
clusters defined by $\check{\mathbf{X}}^{(r^*)}$. This association
can be performed either by using the residual minimization method,
described in SSSC~\cite{SSSC}, or, if subspace dimensions are known, by
identifying the subspace that is closest to each datum, as in
\eqref{eq:findclosestsubspace}.

\begin{remark}
  \label{re:delta_0_update}
  Threshold $\Delta_0$ can be updated across iterations. If
  $\Delta_0^{(0)} = -\infty$ stands for the initial threshold value, $\psi_{\max}^{(r)}$ for the current maximum score as of iteration $r$ and $\psi_{\max}^{(0)} = -\infty$,
  the threshold, at iteration $r\in\{1,\ldots,R_{\max}\}$, is
  updated as  
  \begin{equation}
    \label{eq:Delta0update}
    \Delta_0^{(r)} := \begin{dcases}
      d(\hat{f}^{(r)},\hat{f}_0^{(r)})\,,  &\text{ if }
      d(\hat{f}^{(r)},\hat{f}_0^{(r)}) \ge \Delta_0^{(r-1)} \\ & \text{ and } \psi[d(\hat{f}^{(r)},\tilde{f}^{(r)})]\ge\psi_{\max}^{(r)}\\  
      \Delta_0^{(r-1)}\,, & \text{ otherwise}\,.
    \end{dcases}
  \end{equation}
\end{remark}

\begin{remark}
  As SkeVa-SC realizations are drawn independently, they can be
  readily parallelized using schemes such as
  MapReduce~\cite{dean2008mapreduce}.
\end{remark}

\begin{remark}
  The sketching and validation scheme does not target the SC goal
  explicitly. Rather, it decides whether a sampled subset of data is
  informative or not.
\end{remark}

To estimate the densities involved at each step of SkeVa-SC, kernel
density estimators, or kernel smoothers~[cf. Section
\ref{ssec:kernel}] are employed. Specifically, SkeVa-SC, seeks to
solve the following optimization problem:
$\min_{\check{\mathbf{X}}^{(r)}} d(f,\hat{f}^{(r)})$, where $\hat{f}^{(r)}(\bm{x})
:= (1/n)\sum_{i=1}^{n} K_{\mathbf{H}} (\bm{x}-\bm{x}_i^{(r)})$, and
$\bm{x}_i^{(r)}$ denotes the $i$th column of
$\check{\mathbf{X}}^{(r)}$. As the pdf $f$ to be estimated is
generally unknown, a prudent choice for the bandwidth matrix is
$\mathbf{H} := h^2\mathbf{I}_D$ with $h>0$, as it provides isotropic
smoothing across all dimensions and greatly simplifies the analysis.

To assess performance in closed form, start with the integrated square error
(ISE)
\begin{equation}
  \label{eq:ISE}
  d_{\text{ISE}}(f,g) := \int\left(f(\bm{x}) -
    g(\bm{x})\right)^2 d\bm{x}
\end{equation}
or, as in \cite{skeva}, the Cauchy-Schwarz divergence~\cite{Principe.ITL.10}
\begin{equation}
  \label{eq:CSdiv}
  d_{\text{CS}}(f,g) := -\log\frac{\left(\int
      f(\bm{x})g(\bm{x})d\bm{x}\right)^2}{\int f^2(\bm{x})d\bm{x}\int
    g^2(\bm{x})d\bm{x}}
\end{equation}
 which will henceforth be adopted as our dissimilarity metric $d$. Moreover, we will choose the Gaussian
multivariate kernel, that is $K_{\mathbf{H}}(\bm{x} - \bm{x}_i) =
\phi_{\mathbf{H}}(\bm{x} - \bm{x}_i)$, with $\phi_{\mathbf{H}}$
defined as \eqref{eq:multivariateGaussian}.

\begin{algorithm}[ht]
  \begin{algorithmic}[1]
    \algrenewcommand\algorithmicindent{1em}
    \Require{Data $\mathbf{X}$; max.\ no.\ of iterations $R_{\max}$;
      bandwidth $h$}
    \Ensure{Clustered data; bases of subspaces}

    \For{$r = 1$ to $R_{\max}$}
    \State\parbox[t]{\dimexpr\linewidth-3\dimexpr\algorithmicindent}
    {Sample $n \ll N$ columns of $\mathbf{X}$ to form
      $\check{\mathbf{X}}^{(r)}$.}  
    \State\parbox[t]{\dimexpr\linewidth-3\dimexpr\algorithmicindent}
    {Estimate pdf $\hat{f}^{(r)}$ of 
      $\check{\mathbf{X}}^{(r)}$; evaluate dissimilarity
      $d(\hat{f}^{(r)},\hat{f}_0^{(r)})$ from a unimodal pdf
      $\hat{f}_0^{(r)}$ [cf.~\eqref{eq:pdfdef}].}
    \If{ $d(\hat{f}^{(r)},\hat{f}_0^{(r)}) \geq \Delta_0$} 
    \State\parbox[t]{\dimexpr\linewidth-3\dimexpr\algorithmicindent}
    {Sample $n'$ $(n\leq n'\ll N)$ new columns to form
      $\tilde{\mathbf{X}}^{(r)}$; evaluate pdf $\tilde{f}^{(r)}$ of
      $\tilde{\mathbf{X}}^{(r)}$.}
    \State\parbox[t]{\dimexpr\linewidth-3\dimexpr\algorithmicindent}
    {Find $d(\hat{f}^{(r)},\tilde{f}^{(r)})$ and score
      $\psi[d(\hat{f}^{(r)},\tilde{f}^{(r)})]$.} 
    \EndIf
    \EndFor
    \State{Select winner as $r^* := \arg\max_{r}\psi[d(\hat{f}^{(r)},\tilde{f}^{(r)})]$.}
    \State{Perform SSC (Alg.~\ref{alg:SSC}) on
      $\check{\mathbf{X}}^{(r^{*})}$ to find $\mathbf{\Pi}$.}\label{skeva:ssc} 
    \State{Associate $\mathbf{X}\setminus\check{\bm{X}}^{(r^*)}$ to
      clusters defined in step~\ref{skeva:ssc}.}

  \end{algorithmic}
  \caption{Sketching and validation SC (SkeVa-SC)}\label{alg}
\end{algorithm}

The estimated pdfs that are used by SkeVa-SC are listed as follows:
\begin{subequations}
  \label{eq:pdfdef}
\begin{align}
    & \hat{f}^{(r)}(\bm{x}) := \frac{1}{n}\sum_{i=1}^{n}
    \phi_{\mathbf{H}} \left(\bm{x} - \bm{x}_i^{(r)}\right) \\  
    &\hat{f}_0^{(r)}(\bm{x}) := \phi_{\mathbf{H}_0}\left(\bm{x} -
      \frac{1}{n}\sum_{i=1}^{n}\bm{x}_i^{(r)}\right) \label{def.hat.f}\\  
    & \tilde{f}^{(r)}(\bm{x}) = \frac{1}{n'}\sum_{i=1}^{n'}
    \phi_{\mathbf{H}'} \left(\bm{x} -
      \tilde{\bm{x}}_i^{(r)}\right) \label{def.tilde.f} 
  \end{align}
\end{subequations}
where $\tilde{\bm{x}}_i^{(r)}$ is the $i$th column of
$\tilde{\mathbf{X}}^{(r)}$, and $\mathbf{H}, \mathbf{H}_0, \mathbf{H}'$
are appropriately defined bandwidth matrices. It is then easy to show
that for a measure such as \eqref{eq:ISE} or \eqref{eq:CSdiv} several of the pdf
dissimilarities can be found in closed-form as [cf. \eqref{eq:gaussiantrickmult}]
\begin{subequations}
\begin{alignat}{2}
    & d_{\text{ISE}}(\hat{f}^{(r)},\hat{f}_0^{(r)})  = &&
    \frac{1}{n^2}\sum_{i=1}^{n}\sum_{j=1}^{n} 
    \phi_{2\mathbf{H}}(\bm{x}_i^{(r)} - \bm{x}_j^{(r)}) 
    + \phi_{2\mathbf{H}_0}(\bm{0}) 
    - \frac{2}{n}\sum_{i=1}^{n}\phi_{\mathbf{H} +
      \mathbf{H}_0}\left(\bm{x}_i^{(r)} -
      \frac{1}{n}\sum_{i=1}^{n}\bm{x}_i^{(r)}
    \right) \label{eq:distpp_0} \\
    & d_{\text{ISE}}(\hat{f}^{(r)},\tilde{f}^{(r)}) = &&  
    \frac{1}{n^2}\sum_{i=1}^{n}\sum_{j=1}^{n}\phi_{2\mathbf{H}}(\bm{x}_i^{(r)} 
    - \bm{x}_j^{(r)}) 
     + \frac{1}{n'^2}\sum_{i=1}^{n'}\sum_{j=1}^{n'}
    \phi_{2\mathbf{H}'}(\tilde{\bm{x}}_i^{(r)} - \tilde{\bm{x}}_j^{(r)}) \notag\\
     & && -\frac{2}{nn'}\sum_{i=1}^{n}\sum_{j=1}^{n'}\phi_{\mathbf{H} +
      \mathbf{H}'}(\bm{x}_i^{(r)} -
    \tilde{\bm{x}}_j^{(r)}) \label{eq:distpp_bar}\\    
    & d_{\text{CS}}(\hat{f}^{(r)},\hat{f}_0^{(r)})  = &&
    -2\log\left[\frac{1}{n}\sum_{i=1}^{n}\phi_{\mathbf{H}+\mathbf{H}_0} 
      \left(\bm{x}_i^{(r)} -
        \frac{1}{n}\sum_{j=1}^{n}\bm{x}_i^{(r)}\right)\right]   
     + \log\frac{1}{n^2}\sum_{i=1}^{n}\sum_{j=1}^{n}\phi_{2\mathbf{H}}
    \left(\bm{x}_i^{(r)} - \bm{x}_j^{(r)}\right) \notag\\ 
    & && +\log\phi_{2\mathbf{H}_0}(\bm{0}) \label{eq:distpp_0_CS}\\
    & d_{\text{CS}}(\hat{f}^{(r)},\tilde{f}^{(r)})  = && 
    -2\log\left[\frac{1}{nn'}\sum_{i=1}^{n}\sum_{j=1}^{n'}\phi_{\mathbf{H} 
        + \mathbf{H}'}\left(\bm{x}_i^{(r)} -
        \tilde{\bm{x}}_j^{(r)}\right)\right]    
     + \log\frac{1}{n^2}\sum_{i=1}^{n}\sum_{j=1}^{n}\phi_{2\mathbf{H}}
    \left(\bm{x}_i^{(r)} - \bm{x}_j^{(r)}\right)  \notag\\
    & && + \log\frac{1}{n'^2}\sum_{i=1}^{n'}\sum_{j=1}^{n'}
    \phi_{2\mathbf{H}'}\left(\tilde{\bm{x}}_i^{(r)} -
      \tilde{\bm{x}}_j^{(r)}\right)\,. \label{eq:distpp_bar_CS} 
\end{alignat}
\end{subequations}
Since finding the Gaussian dissimilarities between $n$
$D$-dimensional data incurs complexity $\mathcal{O}(Dn^2)$,
the complexity of Alg.~\ref{alg} is $o[DR_{\max}(n^2 +
nn' + n'^2)]$ per iteration, if the Gaussian kernel is utilized. As
with any algorithm that involves kernel smoothing, the choice of
bandwidth matrices $\mathbf{H},\mathbf{H}',\mathbf{H}_0$ affects critically the performance of Alg.~\ref{alg}.


\section{Performance Analysis}
\label{sec:performance}

The crux of Alg.~\ref{alg} is the $R_{\max}$ number of independent random
draws to identify a subset of data that ``represents well'' the whole data
population. Since SkeVa-SC aims at large-scale clustering, it is natural to
ask whether $R_{\max}$ iterations suffice to mine a subset of data whose
pdf approximates well the unknown $f$. It is therefore crucial, prior to
any implementation of SkeVa-SC, to have an estimate of the minimum number
of $R_{\max}$ that ensures an ``informative draw'' with high probability. Such
concerns are not taken into account in SSSC~\cite{SSSC}, where only a
\textit{single} draw is performed prior to applying SC. This section
provides analysis to establish such a lower-bound on $R_{\max}$,
when the underlying pdf $f$ obeys a Gaussian mixture model (GMM). Due to the universal
approximation properties of
GMM~\cite{bishop2006pattern,universalapprox,gpml}, such a generic
assumption on $f$ is also employed by the mixture of probabilistic
PCA~\cite{tipping1999mixtures} as well as ALC~\cite{ALC2007}.

Performance analysis will be based on the premises that the multimodal data
pdf $f$ is modeled by a mixture of Gaussian densities. This assumption
seems appropriate as any pdf or multivariate function that is $t$-th order
integrable with $t\in[0,\infty)$, can be approximated by a mixture of
appropriately many Gaussians~\cite{universalapprox, wand1994kernel}.

\begin{ass}\label{as:mixturemult}
  Data are generated according to the GMM
  \begin{equation}\label{eq:mixtureofgaussiansmult}
    f(\bm{x}) = \sum_{\ell=1}^{L} w_\ell
    \phi_{\mathbf{\Sigma}_{\ell}}(\bm{x}-\bm{\mu}_{\ell}),\quad
    \sum_{\ell=1}^{L}w_{\ell} = 1
  \end{equation}
  where $L\ge K$, $\bm{\mu}_{\ell}\in\mathbb{R}^{D}$ and
  $\mathbf{\Sigma}_{\ell}\in\mathbb{R}^{D\times D}$ stand for the mean
  vector and the covariance matrix of the $\ell$th Gaussian pdf,
  respectively, and $\{w_{\ell}\}_{l=1}^L \subset [0,1]$ are the
  mixing coefficients.
\end{ass}
Under (As.~\ref{as:mixturemult}), the mean of the entire dataset is
$\bm{\mu}_0 = \sum_{\ell=1}^{L}w_{\ell}\bm{\mu}_{\ell}$, and can be
estimated by the sample mean of all data drawn from $f$. 

\begin{definition}\label{def:distance}
  A ``dissimilarity'' function $d: \mathcal{X}\times\mathcal{X}
  \rightarrow \mathbb{R}$ is a metric or a distance if the
  following properties hold $\forall f_1, f_2, f_3\in\mathcal{X}$:
  \begin{enumerate*}[label={\bf P\arabic*}),ref={\bf P\arabic*}]
  \item\label{prop:semidef} $d(f_1,f_2) \ge 0$; 
  \item\label{prop:coincidence} $d(f_1,f_2) = 0 \Leftrightarrow f_1 = f_2$; 
  \item\label{prop:symmetry} $d(f_1,f_2) = d(f_2,f_1)$; 
  \item\label{prop:triangle} $d(f_1,f_2) \le d(f_1,f_3) + d(f_3,f_2)$.
  \end{enumerate*}
  Property~\ref{prop:triangle} (depicted in
  Fig.~\ref{fig:triangleineq}) is widely known as the triangle
  inequality. A semi-distance is a function $d$ for which~\ref{prop:semidef}, \ref{prop:symmetry},
  \ref{prop:triangle}, and $[d(f,f) = 0, \forall f\in\mathcal{X}]$
  hold. 
  
  A divergence is a function $d$ where $\mathcal{X}$ is the
  space of pdfs, and for which only~\ref{prop:semidef} and
  \ref{prop:coincidence} hold. The class of Bregman divergences are
  generalizations of the squared Euclidean distance and include the
  Kullback-Leibler~\cite{cover2012elements} as well as the
  Itakura-Saito one~\cite{itakurasaito}, among others. Furthermore,
  generalized symmetric Bregman divergences, such as the
  Jensen-Bregman one, satisfy the triangle inequality~\cite{bregmantriangle}. Although
  $d_{\text{ISE}}$ in \eqref{eq:ISE} is not a distance, since it does not
  satisfy the triangle inequality, $\sqrt{d_{\text{ISE}}}$
  (the $L_2$-norm) is a semi-distance function because it satisfies \ref{prop:semidef}, \ref{prop:symmetry},
  \ref{prop:triangle}, and $[d(f,f) = 0, \forall f\in\mathcal{X}]$.
\end{definition}
  
If $\hat{f}$ denotes an estimate of the data pdf $f$, and $f_0$ stands
for a reference pdf (a rigorous definition will follow),
Fig.~\ref{fig:triangleineq} depicts $f, \hat{f}$ and $f_0$ as points
in the statistical manifold $\mathcal{X}$~\cite{statmanifold}, namely
the space of probability distributions. Letting $\delta' := d(f,f_0)$,
the triangle inequality suggests that
  \begin{equation}
    \label{eq:triangleineq}
    |d(f,\hat{f}) - \delta'| \le d(f_0,\hat{f}) \le d(f,\hat{f}) +
    \delta'\,. 
  \end{equation}
  
  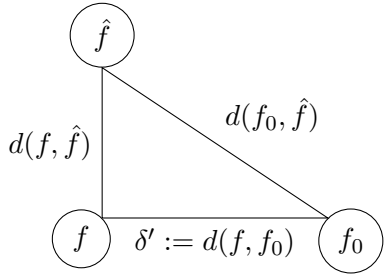
\begin{figure}[h]
    \centering
    \begin{tikzpicture}
      \draw (0,0)
      node[anchor=north east,circle,draw]{$f$} -- (0,2)
      node[midway,auto]{$d(f,\hat{f})$}
      node[anchor=south,circle,draw]{$\hat{f}$}  -- (3,0)
      node[midway,auto]{$d(f_0,\hat{f})$}
      node[anchor=north west,circle,draw]{$f_0$} -- (0,0)
      node[midway,auto]{$\delta' := d(f,f_0)$}; 
    \end{tikzpicture}
    \caption{The triangle inequality in the statistical
      manifold $\mathcal{X}$.}  
          \label{fig:triangleineq}
  \end{figure}

  \begin{definition}\label{def:prbadevent}
    Given a dissimilarity function $d$ which satisfies the triangle
    inequality in Def.~\ref{def:distance}, an event $\mathcal{B}_{\delta}$ per realization of Alg.~\ref{alg}, is
    deemed ``bad'' if the dissimilarity between the kernel-based
    estimator $\hat{f}$ and the true $f$ is larger than or equal
    to some prescribed value $\delta$, i.e.,
    \begin{equation}
      \mathcal{B}_{\delta}:= \{d(\hat{f},f) \geq
      \delta\} \label{def.bad.event}  
    \end{equation}
    where $\{\text{St}(\ldots)\}$ denotes the event of the statement
    $\text{St}(\ldots)$ being true. Naturally, an $\hat{f}$ for which
    the complement $\mathcal{B}_{\delta}^{\complement}$ of
    $\mathcal{B}_{\delta}$ holds true is deemed a ``good''
    estimate of the underlying $f$. Given $f_0\in\mathcal{X}$, and
    $\delta' := d(f,f_0)$, the triangle inequality dictates that $\delta' - d(f_0,\hat{f}) \leq d(f,\hat{f})$. This
    implies that for a fixed $\delta'$, the smaller $d(f_0,\hat{f})$
    is, the larger $d(f,\hat{f})$ becomes. For any arbitrarily fixed
    $\delta_0 \leq \delta' - \delta$, any $\hat{f}$ for which $d(f_0,
    \hat{f}) < \delta_0$ implies that $\delta \leq \delta' - \delta_0
    < \delta' - d(f_0, \hat{f}) \leq d(f,\hat{f})$, i.e.,
    $\mathcal{B}_{\delta}$ occurs [cf.~\eqref{def.bad.event}]. In other words,
    \begin{equation}
      \forall \delta_0\in (0,d(f,f_0) - \delta], \quad
      \{d(f_0,\hat{f})<\delta_0\} \subseteq \mathcal{B}_{\delta}\,. 
      \label{eq:setrelation}
    \end{equation}
    Here, $f_0$ is chosen as the unimodal Gaussian pdf $f_0(\bm{x}) :=
    \phi_{\mathbf{H}_0} (\bm{x} -\bm{\mu}_0)$ that is centered around
    $\bm{\mu}_0$. This is in contrast with the multimodal nature of the
    true $f$, for a number $K>1$ of well-separated
    clusters. According to \eqref{eq:setrelation}, an estimate
    $\hat{f}$ that is ``close'' to the unimodal $f_0$ will give rise to a
    ``bad'' event $\mathcal{B}_{\delta}$.
  \end{definition}
  
  \begin{remark}
    In this section, $f_0$ is centered at the sample mean of the
    entire dataset, while in Alg.~\ref{alg}, $\hat{f}_0$ per realization
    is centered around the sample mean of the sampled data. This is to
    avoid a step that incurs $\mathcal{O}(N)$ complexity in SkeVa-SC. If
    the dataset mean is available (via a preprocessing step) then the
    unimodal pdf per iteration $\hat{f}_0^{(r)}$ can be replaced by $f_0$ induced by the mean of the entire dataset [cf.~\eqref{def.hat.f}].
  \end{remark}

  The maximum required number of iterations $R_{\max}$ can be now
  lower-bounded as summarized in the following theorem. 

  \begin{theorem}\label{th:Rbound}\mbox{}
    \begin{enumerate}
    \item\label{thm:1st.part} Given a distance function $d$
      [cf.~Def.~\ref{def:prbadevent}], a threshold $\delta>0$ [cf.~\eqref{def.bad.event}], a ``success''
      probability $p\in (0,1)$ of Alg.~\ref{alg}, i.e., the
      probability that after $R_{\max}$ realizations a random draw of
      data-points yields an estimate $\hat{f}$ that satisfies
      $\mathcal{B}_{\delta}^{\complement}$
      [cf.~\eqref{def.bad.event}], Alg.~\ref{alg} requires
      \begin{equation}\label{eq:theoremRlowerbound_basic}
        R_{\max} \ge \frac{\log(1-p)}{\log\left(1 -
            \frac{\Expect\left[d(\hat{f},f_0)\right]}{d(f,f_0)-\delta}
          \right)} =: \varrho
      \end{equation}
      where the
      expectation is taken w.r.t.\ the data pdf $f$; that is,
      \begin{equation*}
        \Expect\left[d(\hat{f},f_0)\right] := \int
        d(\hat{f},f_0)f(\bm{x})d\bm{x}\,.
      \end{equation*}
      
    \item\label{thm:2nd.part} Under (As.~\ref{as:mixturemult}) and with $d:= \sqrt{d_{\text{ISE}}}$
      [cf.~\eqref{eq:ISE}], an overestimate $\hat{\varrho}$ with prescribed
      probability $1-q$ of the lower bound $\varrho$ in
      \eqref{eq:theoremRlowerbound_basic} is given by
      \begin{equation}
        \label{eq:theoremRlowerbound} \hat{\varrho} :=
        \frac{\log(1-p)}{\log\left(1 - \frac{\Expect[d_{\text{ISE}}(\hat{f},f_0)]}{\left(\theta_1 +
              \theta_2\right)^2}\right)}
      \end{equation}
      with the expected value expressed in closed form as
      \begin{subequations}
        \begin{alignat}{2}  
          &\Expect[d_{\text{ISE}}(\hat{f},f_0)] = 
           \frac{1}{(4\pi)^{D/2}|\mathbf{H}_0|^{1/2}} +
          \frac{1}{n}\frac{1}{(4\pi)^{D/2}|\mathbf{H}|^{1/2}} 
           + \left(1 -
            \frac{1}{n}\right)\mathbf{w}^{\top}\mathbf{\Omega}_2\mathbf{w} 
          \notag\\
          & \hspace{80pt} - 2\sum_{\ell=1}^{L}w_{\ell}\phi_{\mathbf{H}+\mathbf{H}_0 +
            \mathbf{\Sigma}_{\ell}}(\bm{\mu}_{\ell} -
          \bm{\mu}_0) \label{eq:zeta2}\\
          & \theta_1  := 
          \left[-\frac{2\log(q/2)}{nh\left(4\pi\right)^{D/2}}\right]^{1/2} + 
          \left\{\frac{1}{n(4\pi)^{D/2}|\mathbf{H}|^{1/2}}\right. 
           + \left. \mathbf{w}^{\top}\left[\left(1-
                \frac{1}{n}\right)\mathbf{\Omega}_2 - 2\mathbf{\Omega}_1 +
              \mathbf{\Omega}_0\right]\mathbf{w}
          \right\}^{1/2}\label{eq:theta1}\\
          & \theta_2  :=  \left\{
            \mathbf{w}^{\top}\mathbf{\Omega}_0\mathbf{w} + 
            \frac{1}{(4\pi)^{D/2}|\mathbf{H}_0|^{1/2}} \right. 
           \left. - 2\sum_{\ell=1}^{L}w_{\ell}\phi_{\mathbf{\Sigma}_{\ell} +
              \mathbf{H}_0}(\bm{\mu}_{\ell} - \bm{\mu_0}) \right\}^{1/2} 
        \end{alignat}
      \end{subequations}
      where $\mathbf{w} := [w_1,w_2,\ldots,w_L]^{\top}$ is the vector
      of mixing coefficients of
      \eqref{eq:mixtureofgaussiansmult} and
      $\mathbf{\Omega}_\alpha\in\mathbb{R}^{L\times L},
      \alpha\in\{0,1,2\}$, is a matrix with $(i,j)$ entry
      \begin{equation}
        \label{eq:omegamultivariate} [\mathbf{\Omega}_{\alpha}]_{ij} =
        \phi_{\alpha\mathbf{H}+\mathbf{\Sigma}_i +
          \mathbf{\Sigma}_j}(\bm{\mu}_i - \bm{\mu}_j)\,.
      \end{equation}
    \end{enumerate}
  \end{theorem}

\begin{proof}
  By definition, $1-p$ is the probability that Alg.~\ref{alg} yields
  $R_{\max}$ ``bad'' draws. Since iterations are independent, it holds
  that
  \begin{equation}
    \label{eq:GlobalBadEvent}
    [\prob(\mathcal{B}_{\delta})]^{R_{\max}} = 1-p\,.
  \end{equation}
  The number of draws $R_{\max}$ can be lower-bounded as 
  \begin{align}
      [\prob (\mathcal{B}_{\delta})]^{R_{\max}} \le 1-p  \Leftrightarrow
      R_{\max}\log\left(\prob(\mathcal{B}_{\delta}) \right) \le 
      \log(1-p) 
       \Leftrightarrow  R_{\max} \ge
      \frac{\log(1-p)}{\log\left(\prob(\mathcal{B}_{\delta})
        \right)} \label{eq:basicRlowerbound} 
  \end{align}
  where the last inequality follows from the trivial fact that
  $\prob(\mathcal{B}_{\delta})<1 \Leftrightarrow
  \log\prob(\mathcal{B}_{\delta}) < 0$.

  Using \eqref{eq:setrelation}, $\prob(\mathcal{B}_{\delta})$ is
  lower-bounded as
  \begin{align}
    \prob(\mathcal{B}_{\delta}) & = \prob\left(d(\hat{f},f)
      \geq\delta\right)  
     \ge \prob\left(d(\hat{f},f_0) <
      \delta_0\right) \label{eq:basicbadevent} 
  \end{align}
  where $\delta_0 := d(f,f_0)-\delta$; thus, when $\delta$ is
  fixed, so is $\delta_0$. Furthermore, Markov's 
  inequality implies that
  \begin{align}
      \prob\left(d(\hat{f},f_0) < \delta_0\right)  = 1 -
      \prob\left(d(\hat{f},f_0) \geq \delta_0\right) 
       \ge 1 - \frac{\Expect\left[d(\hat{f},f_0)\right]}{\delta_0}\,.
    \label{eq:markovd0}
  \end{align}
  Combining \eqref{eq:basicbadevent} with
\eqref{eq:markovd0} yields
  \begin{equation}\label{eq:prbadgeexpect}
    \prob(\mathcal{B}_{\delta}) \ge 1 -
    \frac{\Expect\left[d(\hat{f},f_0)\right]}{\delta_0}\,.   
  \end{equation}
  Using now \eqref{eq:prbadgeexpect} and that $\log(1-p)< 0$, it is
  easy to establish \eqref{eq:theoremRlowerbound_basic} via
  \eqref{eq:basicRlowerbound}. In the case 
  where $1 - {\Expect[d(\hat{f},f_0)]}/{\delta_0} \leq 0$,
  \eqref{eq:prbadgeexpect} is uninformative, and $1$ is used as the
  trivial lower bound of $R_{\max}$. This completes the proof of
  Thm.~\ref{th:Rbound}.\ref{thm:1st.part}.

  Regarding the proof to establish \eqref{eq:theoremRlowerbound},
  closed-form expressions of $\Expect[d(\hat{f},f_0)]$ as well as values
  for $\delta$, $\delta_0$ and $\delta'$, using
  $d_{\text{ISE}}(\cdot,\cdot)$, are provided in Appendix
  \ref{app:boundsISE}.
\end{proof}

Fig.~\ref{fig:Rmax} depicts $\hat{\varrho}$ evaluated using \eqref{eq:theoremRlowerbound} as $n$ increases for a synthetic one-dimensional ($D=1$)
dataset with $N=480$ and $K=3$ clusters generated according to
\eqref{eq:mixtureofgaussiansmult}. The cluster means are $\{0,0.5,1\}$,
the variances are $\{0.3,0.3,0.3\}$, and the number of data per
cluster are $\{100,180,200\}$, respectively. Using $(p,q):= (0.99, 0.01)$ in
Thm.~\ref{th:Rbound}.\ref{thm:2nd.part}, the results shown in
Fig.~\ref{fig:Rmax} are intuitively pleasing: As the number of sampled
points $n$ increases, the required number of random draws decreases,
and at $n=480$ only one draw is required.
 
For the same dataset, Fig.~\ref{fig:AccRmax} depicts the accuracy (\%
of correctly clustered data [cf.~Section~\ref{sec:numerical_tests}])
of Alg.~\ref{alg} using $K$-means clustering instead of SSC, as
$R_{\max}$ increases. Here, the number of sampled data is fixed to
$n=10$ while the number of data for the validation phase is set to
$n'=50$. Alg.~\ref{alg} is compared to the simple scheme of taking a
\textit{single} random draw of $n$ data and performing $K$-means. The
vertical red line in Fig.~\ref{fig:AccRmax} indicates the value of the overestimate
$\hat{\varrho}$ provided by Thm.~\ref{th:Rbound}.\ref{thm:2nd.part}. In this case, the overestimate suggests roughly $10-15$ realizations above the ``knee'' where clustering performance of Alg.~\ref{alg} improves over the performance of a single random draw. The
results are averaged over $10$ independent Monte Carlo runs.


\begin{figure}[h]
\centering
\includegraphics[width=0.7\columnwidth]{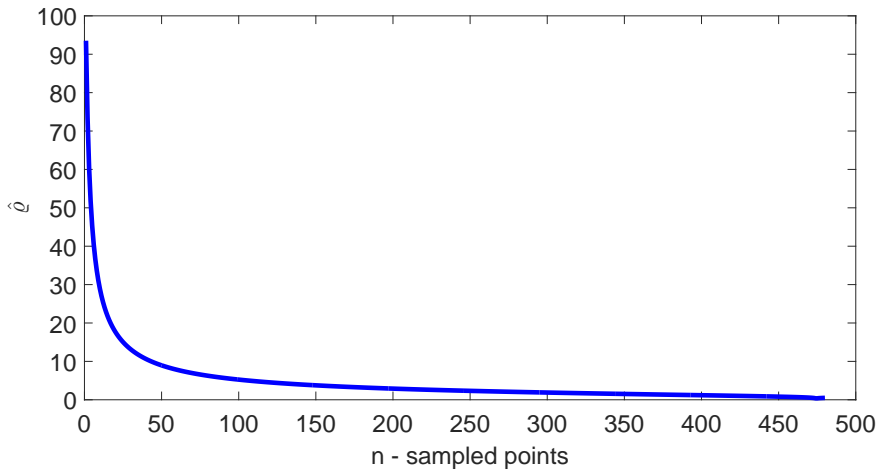}
\caption{Values of $\hat{\varrho}$ versus number of sampled points $n$ for
  $K=3$ clusters generated by \eqref{eq:mixtureofgaussiansmult} with
  $D=1$, cluster means $\{0,0.5,1\}$,  variances $\{0.3,0.3,0.3\}$, and
  number of points per cluster $\{100,180,200\}$.}\label{fig:Rmax} 
\end{figure}

\begin{figure}[h]
\centering
\includegraphics[width=0.7\columnwidth]{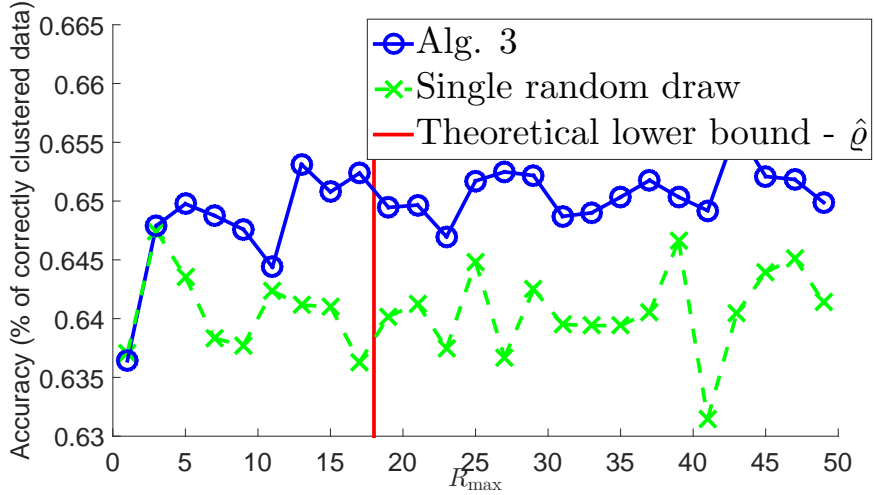}
\caption{Accuracy (\% of correctly clustered data) versus $R_{\max}$
  number of sampled points $n$ for the data used in
  Fig.~\ref{fig:Rmax}. Alg.~\ref{alg} is compared to the simple scheme
  of taking a single random draw of $n$ data, followed by
  clustering. The vertical line represents the $\hat{\varrho}$
  value given by Theorem~\ref{th:Rbound}.} \label{fig:AccRmax}
\end{figure}

When reliable estimates of $f$ are unknown, values for $R_{\max}$ can
be estimated on-the-fly using \eqref{eq:theoremRlowerbound_basic}
and sample averaging. The ensemble averages of
Thm.~\ref{th:Rbound}.\ref{thm:1st.part} can be replaced by sampled
ones, where averaging takes place across iterations. Specifically,
$\Expect[d(\hat{f},f_0)]$ [cf.~\eqref{eq:theoremRlowerbound_basic}] is
replaced by $\bar{d}^{(r)}(\hat{f},f_0)$, which denotes the running sample
average of $d(\hat{f},f_0)$, as of realization $r$. Moreover, since the
number $n'$ of data used in the validation phase is set to be larger
than or equal to the number $n$ of data drawn in the sketching phase
of Alg.~\ref{alg}, $\tilde{f}$ of \eqref{def.tilde.f} provides
potentially better estimates of $f$ than $\hat{f}$ does. As a result,
distance $d(\tilde{f}^{(i)},f_0^{(i)})$ can be employed as a surrogate to
$\delta'$ [cf.~\eqref{eq:triangleineq}]. Further, by $d:=
\sqrt{d_{\text{ISE}}}$ and by using \eqref{eq:triangle_app},
\eqref{eq:deltabound}, as well as replacing ensemble with
sample averages, an estimate of $\delta_0$ at realization $r$ of
Alg.~\ref{alg} is
\begin{equation}
\bar{\delta}_0^{(r)} := \left(
\sqrt{-\frac{2\log(q/2)}{nh\left(4\pi\right)^{D/2}}} +
\bar{d}^{(r)}(\tilde{f},\hat{f}) + \bar{d}^{(r)}(\tilde{f},f_0)\right)^2
\label{eq:bardelta_0}
\end{equation}
where $\bar{d}^{(r)}(\tilde{f},\hat{f})$ and
$\bar{d}^{(r)}(\tilde{f},f_0)$ are the sample averages of
$d(\tilde{f},\hat{f})$ and $d(\tilde{f},f_0)$,
respectively. Therefore, according to \eqref{eq:pdfdef}, the following
quantities can be computed per realization of Alg.~\ref{alg}:
\begin{subequations}
\begin{alignat}{1}
 \bar{d}^{(r)}(f_0,\hat{f})  = 
\frac{1}{r}\sum_{i=1}^{r}d(f_0^{(i)},\hat{f}^{(i)})  = \frac{r-1}{r}\bar{d}^{(r-1)}(f_0,\hat{f}) + \frac{1}{r}d(f_0^{(r)},\hat{f}^{(r)})
\label{eq:sample_df0fhat}
\end{alignat}
\begin{alignat}{1}
 \bar{d}^{(r)}(\tilde{f},\hat{f})  =
\frac{1}{r}\sum_{i=1}^{r}d(\tilde{f}^{(i)},\hat{f}^{(i)})  =
\frac{r-1}{r}\bar{d}^{(r-1)}(\tilde{f},\hat{f}) + \frac{1}{r}d(\tilde{f}^{(r)},\hat{f}^{(r)})
\label{eq:sample_dffhat}
\end{alignat}
\begin{alignat}{1}
 \bar{d}^{(r)}(\tilde{f},f_0)  =
\frac{1}{r}\sum_{i=1}^{r}d(\tilde{f}^{(i)},f_0^{(i)})  =
\frac{r-1}{r}\bar{d}^{(r-1)}(\tilde{f},f_0) + \frac{1}{r}d(\tilde{f}^{(r)},f_0^{(r)})\,.
\label{eq:sample_dff0}
\end{alignat}
\end{subequations}
Consequently, \eqref{eq:theoremRlowerbound_basic} can be approximated
per iteration $r$ using \eqref{eq:bardelta_0} and
\eqref{eq:sample_df0fhat}-\eqref{eq:sample_dff0} as
\begin{equation}
\boxed{\hat{R}_{\max}^{(r)} \ge \frac{\log(1-p)}{\log\left(1 -
    \frac{\bar{d}^{(r)}(f_0,\hat{f})}{\bar{\delta}_0^{(r)}}\right)} =:
\bar{\varrho}^{(r)}.}
\label{eq:estimatedRmax}
\end{equation}
Based on the previous sample averages, a simple rule for updating
$\hat{R}^{(r)}_{\max}$, which quantifies the maximum number of
iterations to be executed in Alg.~\ref{alg}, as of iteration $r$, is
 \begin{equation}
 \hat{R}^{(r)}_{\max} := \max\left\{\bar{\varrho}^{(r)}, R_0
 \right\}\label{eq:Rmaxrule}
 \end{equation}
 where $R_0$ is a prescribed absolute minimum number of iterations.
 Fig.~\ref{fig:hatRmax} shows the values of \eqref{eq:Rmaxrule} as
 iterations of Alg.~\ref{alg} are run for the dataset used in
 Figs.~\ref{fig:Rmax} and \ref{fig:AccRmax}. Here $R_0 := 3$, and
 $\hat{R}_{\max}^{(r)}$ approaches the theoretical value given by
 Thm.~\ref{th:Rbound}, as values of the distances are averaged across
 iterations.

\begin{figure}[h]
\centering
\includegraphics[width=0.7\columnwidth]{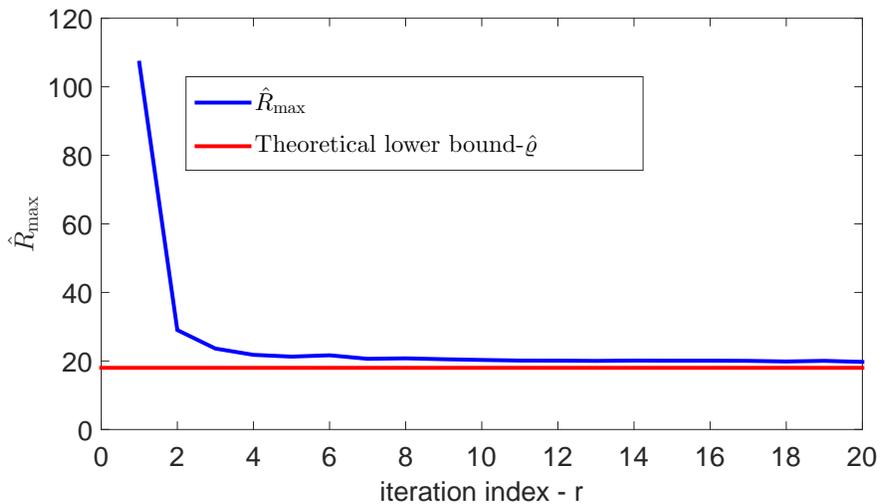}
\caption{Estimates
  $\hat{R}^{(r)}_{\max}$ [cf.~\eqref{eq:Rmaxrule}] versus iteration index
  for a dataset with $K=3$ clusters generated by
  \eqref{eq:mixtureofgaussiansmult} with $D=1$, cluster means
  $\{0,0.5,1\}$, variances $\{0.3,0.3,0.3\}$, and number of points per
  cluster $\{100,180,200\}$. The horizontal line represents the
  $\hat{\varrho}$ value given by Theorem~\ref{th:Rbound} when $f$ is
  known.} \label{fig:hatRmax}
\end{figure}

\section{Numerical Tests}\label{sec:numerical_tests}

The proposed algorithm is validated using synthetic and real
datasets. SkeVa-SC is compared to SSSC, which is the
state-of-the-art algorithm for scalable SC. The metrics
evaluated are:
%
\begin{itemize}
\item Accuracy, i.e., percentage of correctly clustered data:
\begin{equation*}
\text{Accuracy} := \frac{\text{\# of data correctly clustered}}{N}\,.
\end{equation*}

\item Normalized mutual information (NMI)~\cite{cai2005document}
  between experimental and the ground truth labels: 
\begin{equation*}
\text{NMI} := \frac{I(\varPi;\varPi')}{\max\{\mathcal{H}(\varPi),\mathcal{H}(\varPi')\}} 
\end{equation*}
where $\varPi$ is a random variable taking values $\{1,2,\ldots,K\}$
with probabilities $p(\pi_k) = \Pr\left(\varPi = k\right) =
{N_k}/{N}$ of a datum to belong to cluster $k$. Here $N_k$ is the
number of data in cluster $k$, and its value is provided by the
algorithms tested. Quantity $\varPi'$ is a random variable identical
to $\varPi$, but with probabilities derived by the ground-truth
labels, while $I(\varPi;\varPi')$ is the mutual information between
the random variables $\varPi$ and $\varPi'$~\cite{cover2012elements}:
\begin{equation*}
I(\varPi;\varPi') := \sum_{i,j}
p(\pi_i,\pi_j')\log\frac{p(\pi_i,\pi_j')}{p(\pi_i)p(\pi_j')}\,,
\end{equation*}
where $\mathcal{H}(\varPi)$ denotes the entropy of
$\varPi$ defined as~\cite{cover2012elements}
\begin{equation*}
\mathcal{H}(\varPi) := -\sum_{i}p(\pi_i)\log p(\pi_i)\,.
\end{equation*}
\item Time (in seconds) required for clustering all data.
\end{itemize} 
The bandwidth matrix used for Alg.~\ref{alg} is $\mathbf{H} :=
h^2(C,n) \mathbf{I}_D$ with $h$ similar to that in
\eqref{eq:optASbw}, where the unknown quantity pertaining to the
curvature of $f$ is replaced by a user-defined constant $C$; that is
\begin{equation}
\label{eq:bwused}
\begin{aligned}
h(C,n) := \left[\frac{CD\phi_{2\mathbf{I}}(\bm{0})}{n}\right]^{1/(D+4)}
= \left[\frac{CD}{n\left(4\pi\right)^{D/2}}\right]^{1/(D+4)} \,.
\end{aligned}
\end{equation}
Furthermore, the scoring function $\psi$ used in Alg.~\ref{alg} is
 $\psi(x) := x^{-1}$. The software used to conduct all experiments is MATLAB~\cite{MATLAB:2014}. All results represent the
averages of $10$ independent Monte Carlo runs.

\subsection{Synthetic data}

Tests on synthetic datasets of dimensions $D=100$ and $D=500$ with
$K=5$ subspaces are shown in Figs.~\ref{fig:results} and
\ref{fig:results500}, respectively. Subspaces have dimensions
$\{12,10,5,3,2\}$, and the number of data per subspace is proportional
to the subspace dimension: $N_k := 200d_k$. The data per subspace have
been generated according to \eqref{eq:pointsubspace}, where
$\{\bm{m}_k = \bm{0}\}_{k=1}^{K}$, the subspace bases
$\{\mathbf{U}_k\}_{k=1}^{K}$ are randomly generated so that the
subspace angle $\eta(i,j)$ between $\mathcal{S}_i,\mathcal{S}_j$ given by
\begin{equation*}
\eta(i,j) := \min_{\bm{u},\bm{v}}
\left \{\arc\cos
\left(\frac{|\bm{u}^{\top}\bm{v}|}{\|\bm{u}\|\|\bm{v}\|} \right) :
\bm{u}\in{\mathcal{S}}_i,  \bm{v}\in{\mathcal{S}}_j \right\}
\end{equation*}
is at least $\pi/4$, $\forall(i,j) \in \{1, \ldots, K\}^2$, $i\neq
j$.
Projections $\bm{y}_i$ of $\bm{x}_i$'s onto their subspaces are also
randomly drawn from a uniform distribution on the volume of the unit
$d_k$-dimensional hypercube: $\bm{y}_i \sim \mathcal{U}[-1,1]^{d_k}$, where
$\mathcal{U}$ denotes uniform distribution. Finally, AWGN with
variance $\sigma^2 = 0.1$ is added to all data:
$\bm{v}_i\sim\mathcal{N}(\bm{0},\sigma^2\mathbf{I}_D)$. The number of
points for the validation stage of Alg.~\ref{alg} is $n'=600$ and the
maximum number of iterations is set to $R_{\max}=100$. As the number of
sampled data $n$ increases, so does the clustering accuracy, as well as the
normalized mutual information (NMI) between the cluster assignments and the
ground truth ones. Both of these metrics for the proposed method are larger
than those of SSSC, corroborating the fact that a single random sketch
(SSSC) may not be representative of the data to ensure reliable
clustering. Additionally, required processing times by SSSC and SkeVa-SC
are comparable.

\begin{figure}[htb]
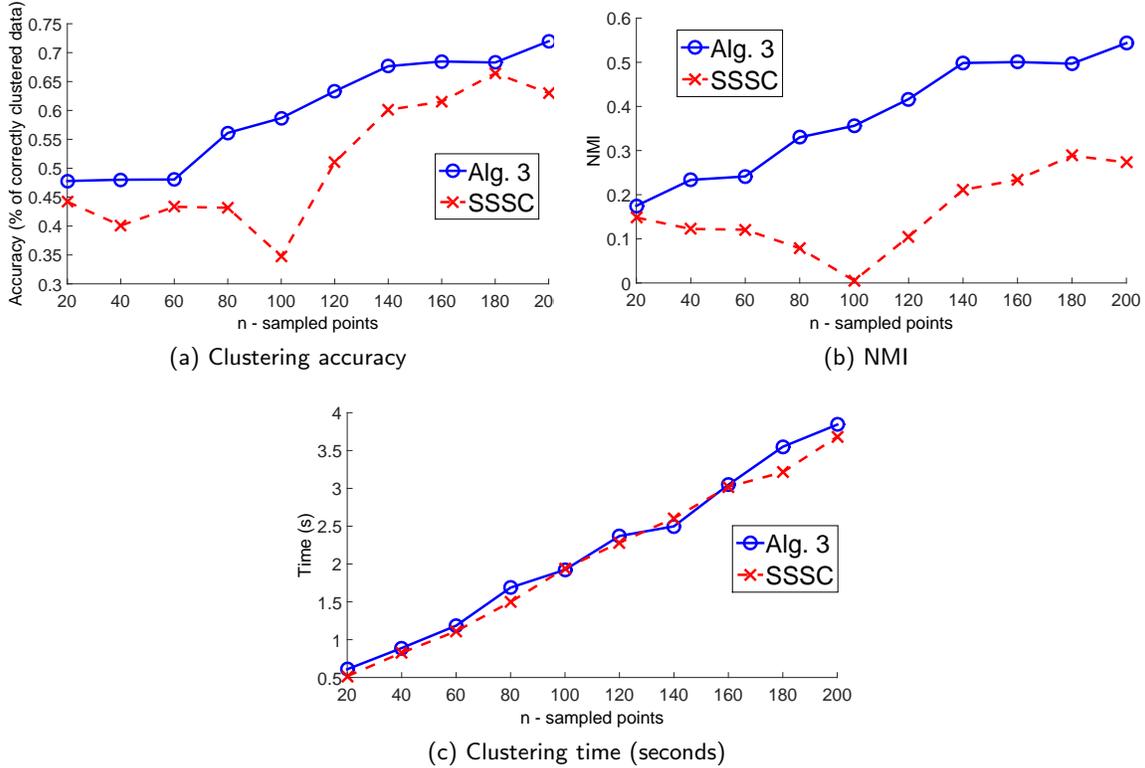

        \centering
        \begin{subfloat}[Clustering
          accuracy]{\includegraphics[width=0.45\columnwidth]{accuracy.eps} 
                        \label{fig:accuracy}}
        \end{subfloat}
        \begin{subfloat}[NMI]{\includegraphics[width=0.45\columnwidth]{NMI.eps} 
                        \label{fig:nmi}}
        \end{subfloat}
        
        \begin{subfloat}[Clustering time
          (seconds)]{\includegraphics[width=0.45\columnwidth]{time.eps} 
                        \label{fig:time}}
        \end{subfloat}
        \caption{Simulated tests on a synthetic dataset with $K=5$
          subspaces, $D=100$ and $N=6,400$
          data.}\label{fig:results}
\end{figure}

\begin{figure}[htb]
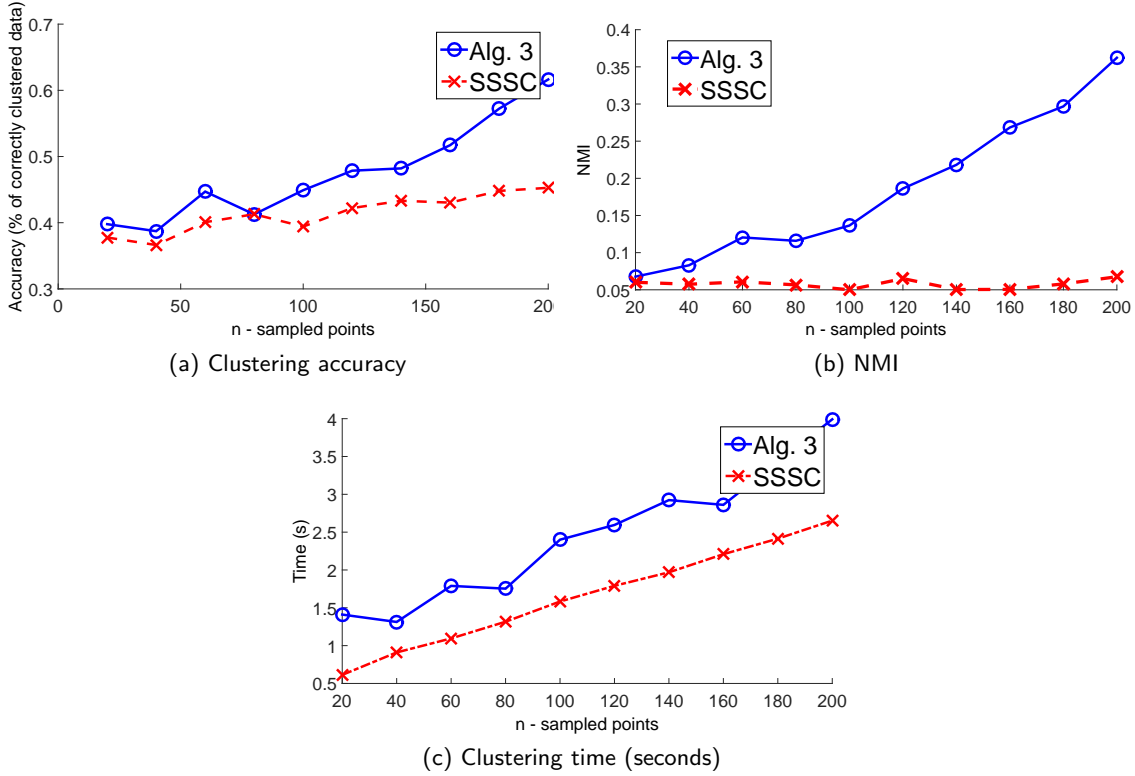

        \centering
        \begin{subfloat}[Clustering
          accuracy]{\includegraphics[width=0.45\columnwidth]{acc-500.eps} 
                        \label{fig:accuracy500}}
        \end{subfloat}%
        \begin{subfloat}[NMI]{\includegraphics[width=0.45\columnwidth]{nmi-500.eps}  
                        \label{fig:nmi500}}
        \end{subfloat}

        \begin{subfloat}[Clustering time
          (seconds)]{\includegraphics[width=0.45\columnwidth]{time-500.eps} 
                        \label{fig:time500}}
        \end{subfloat}
        \caption{Simulated tests on a synthetic dataset with $K=5$ subspaces, 
        $D=500$ and $N=6,400$ data.}\label{fig:results500} 
\end{figure}

\subsection{Real data}

The real datasets tested are the PenDigits~\cite{pendigits}, the
Extended Yale Face ~\cite{yaleb}, and the PokerHand UCI~\cite{UCI}
databases. The PenDigits dataset includes $N=10,992$ data of dimension
$D=16$, separated into $K=10$ clusters, with each datum representing a
handwritten digit. Clusters group same digits, and each cluster
contains $N_k = 250$ data.

The results for the PenDigits dataset are shown in
Fig.~\ref{fig:pendigits}, with $C=10^{-3}$ [cf.~\eqref{eq:bwused}],
$\mathbf{H} = h^2(C,n)\mathbf{I}_D, \mathbf{H}_0 =
({h^2(C,n)}/{4})\mathbf{I}_D, \mathbf{H}' = h^2(C,n')\mathbf{I}_D$,
$n' = 700$, and $R_{\max} = 150$. Similar to the synthetic tests, as the
number of data increases so does the accuracy and NMI of both
Alg.~\ref{alg} and SSSC, with Alg.~\ref{alg} showing higher accuracy
and NMI levels at the cost of higher computational time. The accuracy
and NMI difference between SSSC and Alg.~\ref{alg} is not as
pronounced as in the synthetic datasets, possibly because the
PenDigits dataset is uniform, that is all clusters have the same
number of data.

\begin{figure}[htb]
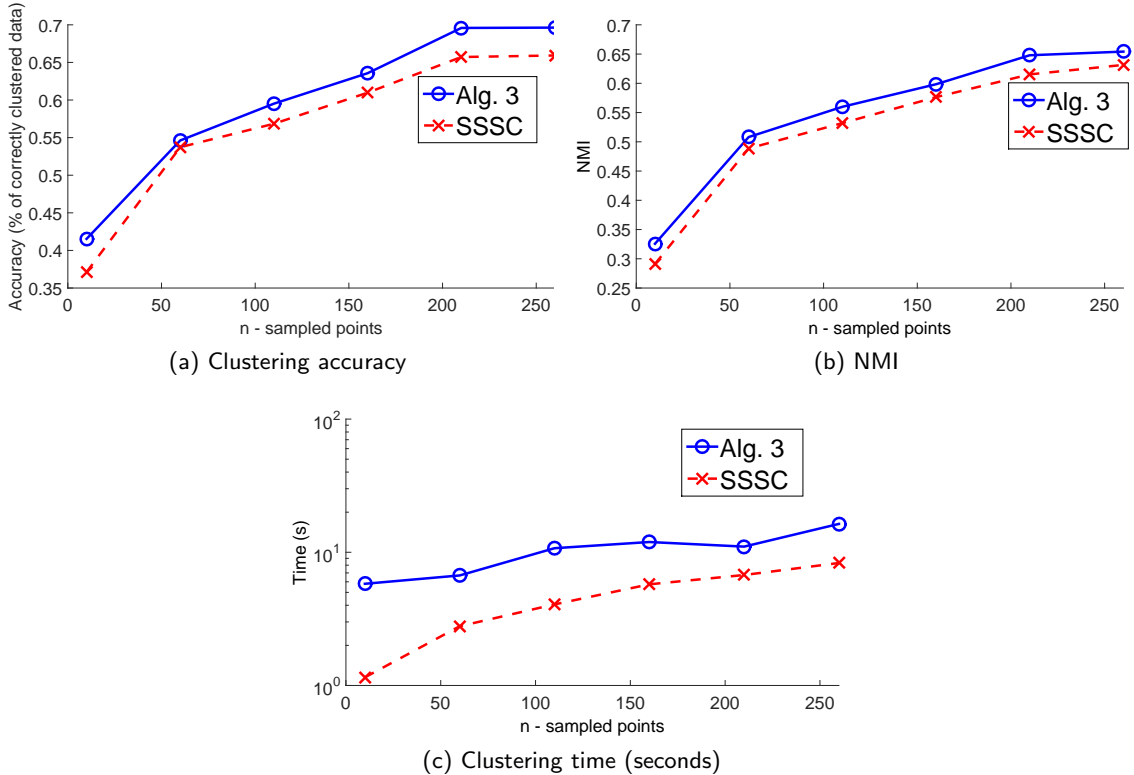
 
        \centering
        \begin{subfloat}[Clustering
          accuracy]{\includegraphics[width=0.45\columnwidth]{pendigitsacc.eps} 
                        \label{fig:pendigitsacc}}
        \end{subfloat}%
        \begin{subfloat}[NMI]{\includegraphics[width=0.45\columnwidth]{pendigitsnmi.eps}  
                        \label{fig:pendigitsnmi}}
        \end{subfloat}

        \begin{subfloat}[Clustering time
          (seconds)]{\includegraphics[width=0.45\columnwidth]{pendigitstime.eps} 
                        \label{fig:pendigitstime}}
        \end{subfloat}
        \caption{Simulated tests on real dataset PenDigits, with
          $N=10,992$ data dimension $D=16$ and $K=10$
          clusters.}\label{fig:pendigits} 
\end{figure}

The Extended Yale Face database contains $N=2,414$ face images of
$K=38$ people, each of dimension $D=2,016$. The dimensionality of the
data was reduced using PCA by extracting the $114$ most important
features and Alg.~\ref{alg} and SSSC were tested on the dimensionality
reduced and normalized data. Fig.~\ref{fig:yale} shows the results for
this dataset, with $C = 10^{-2}$, $\mathbf{H} = h^2(C,n)\mathbf{I}_D,
\mathbf{H}_0 = ({h^2(C,n)}/{4})\mathbf{I}_D, \mathbf{H}' =
{h^2(C,n')}\mathbf{I}$, $n' = 700$, and $R_{\max} = 150$. Again
Alg.~\ref{alg} exhibits higher accuracy and NMI than its one-shot
random sampling counterpart SSSC at basically comparable clustering
time.

Furthermore, Figs.~\ref{fig:yaleacc_a} and \ref{fig:yaletime_a} show
accuracy and time results for the Yale Face database when the maximum
number of iterations is estimated on-the-fly, as described in
Section~\ref{sec:performance}, using
\eqref{eq:theoremRlowerbound_basic} after replacing the ensemble
averages with sample averages across
iterations. Fig.~\ref{fig:yalermax_a} shows the values of
$\hat{R}_{\max}$ versus the number of sampled data. Here $C=1$,
$\mathbf{H} = h^2(C,n)\mathbf{I}_D, \mathbf{H}_0 =
({h^2(C,n)}/{4})\mathbf{I}_D, \mathbf{H}' = ({h(C,n')})^2\mathbf{I}$,
$n' = 600$, $p=0.99$ and $q=0.01$. Similarly to the prescribed
$R_{\max}$ case Alg.~\ref{alg} exhibits higher clustering accuracy
while requiring comparable running time as SSSC.

\begin{figure}[htb]
        \centering
        \begin{subfloat}[Clustering
          accuracy]{\includegraphics[width=0.45\columnwidth]{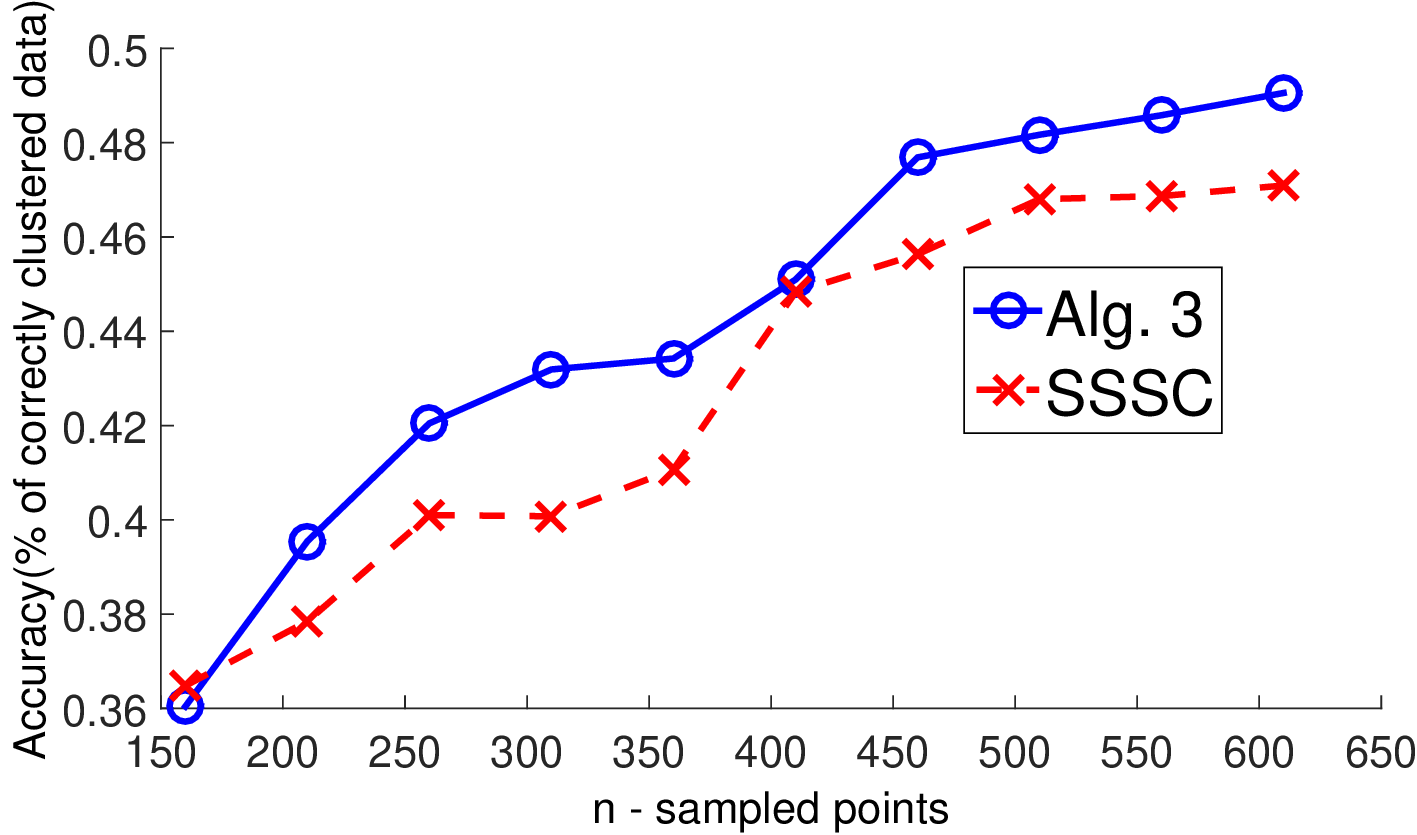} 
                        \label{fig:yaleacc}}
        \end{subfloat}%
        \begin{subfloat}[NMI]{\includegraphics[width=0.45\columnwidth]{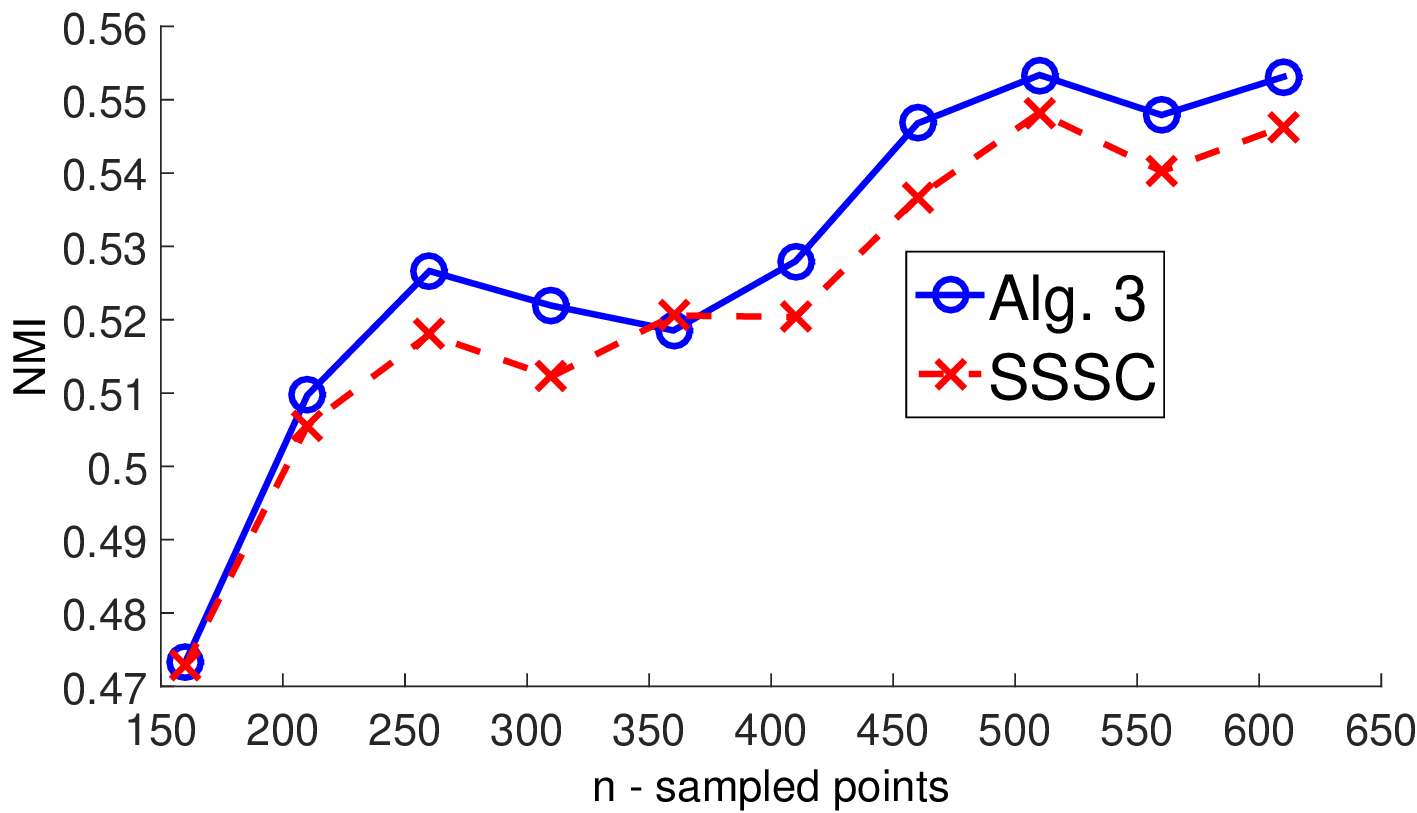}
                        \label{fig:yalenmi}}
        \end{subfloat}

        \begin{subfloat}[Clustering time
          (seconds)]{\includegraphics[width=0.45\columnwidth]{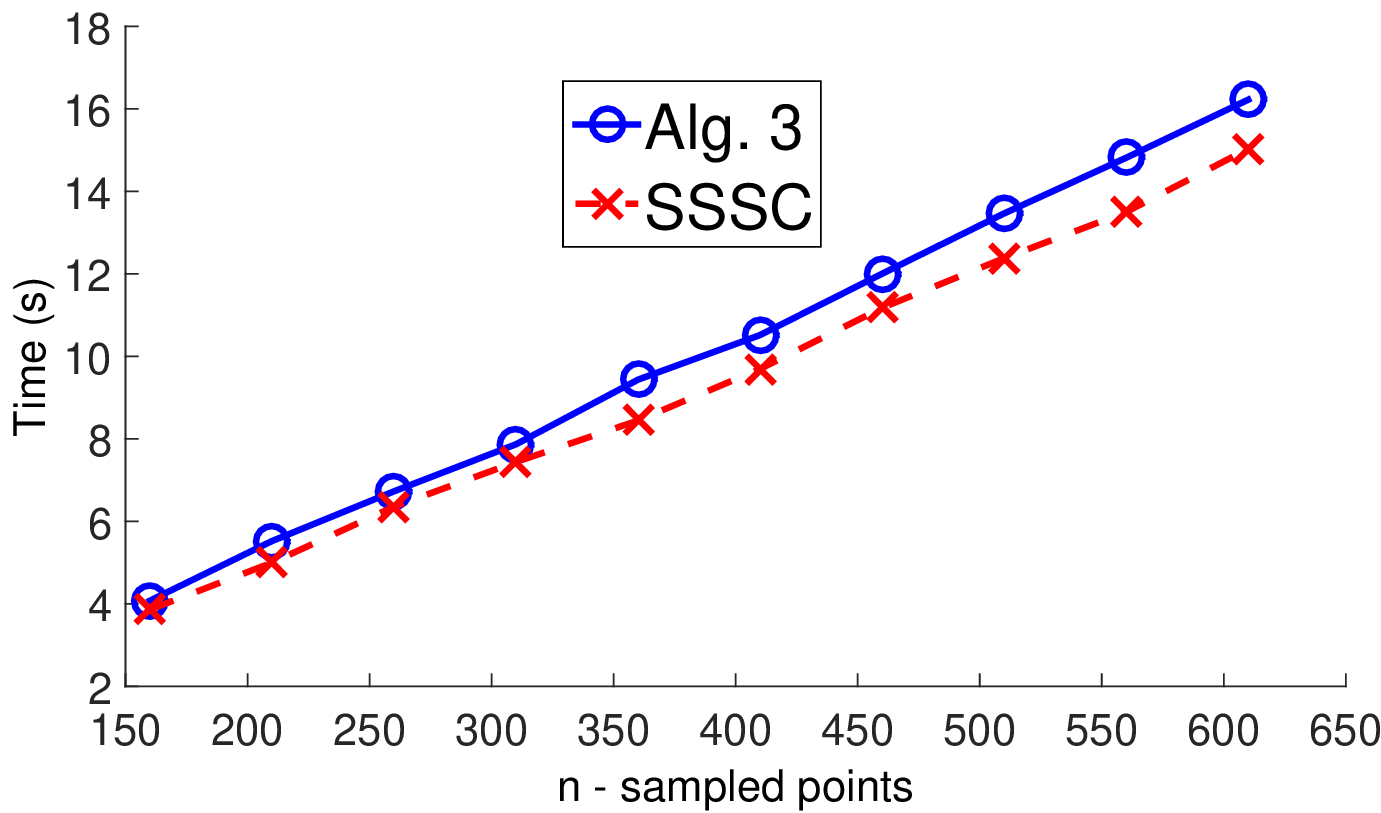} 
                        \label{fig:yaletime}}
        \end{subfloat}
        \caption{Simulated tests on real dataset Extended Yale Face
          Database B, with $N=2,414$ data dimension $D=114$ and
          $K=38$ clusters.}\label{fig:yale} 
\end{figure}

\begin{figure}[htb]
        \centering
        \begin{subfloat}[Clustering
          accuracy]{\includegraphics[width=0.45\columnwidth]{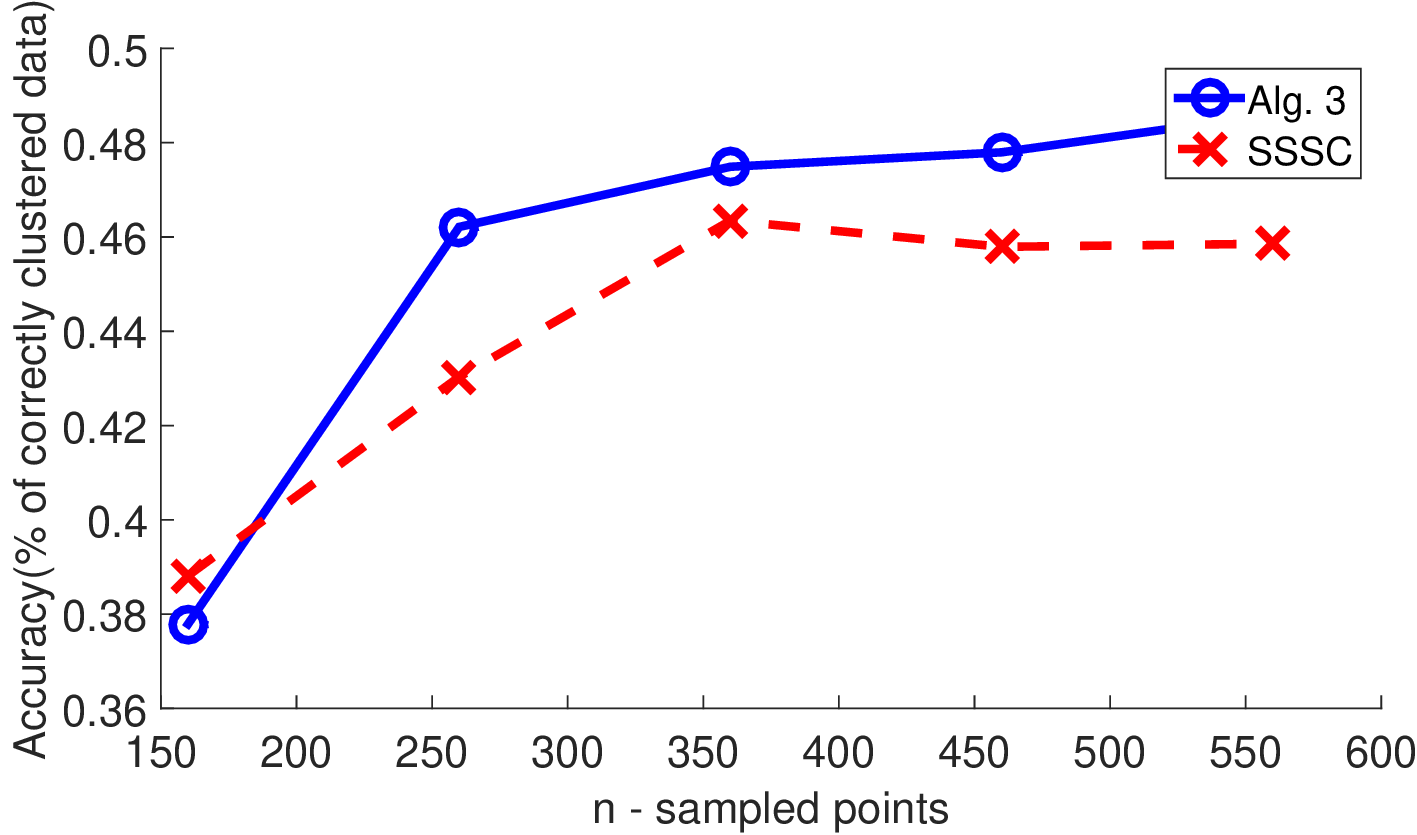} 
                        \label{fig:yaleacc_a}}
        \end{subfloat}%
        \begin{subfloat}[Clustering time (seconds)]{\includegraphics[width=0.45\columnwidth]{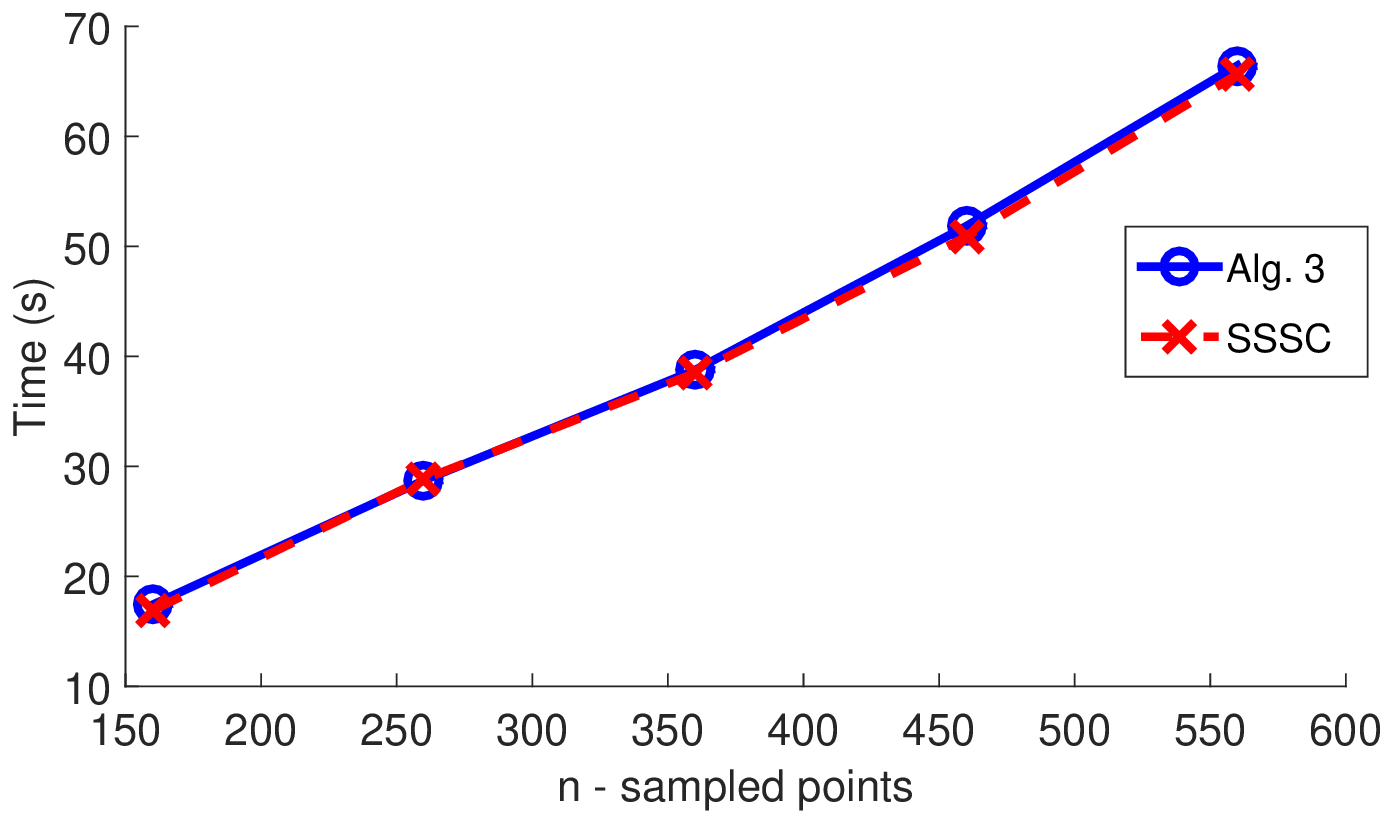}
                        \label{fig:yaletime_a}}
        \end{subfloat}

        \begin{subfloat}[$\hat{R}_{\max}$]{\includegraphics[width=0.45\columnwidth]{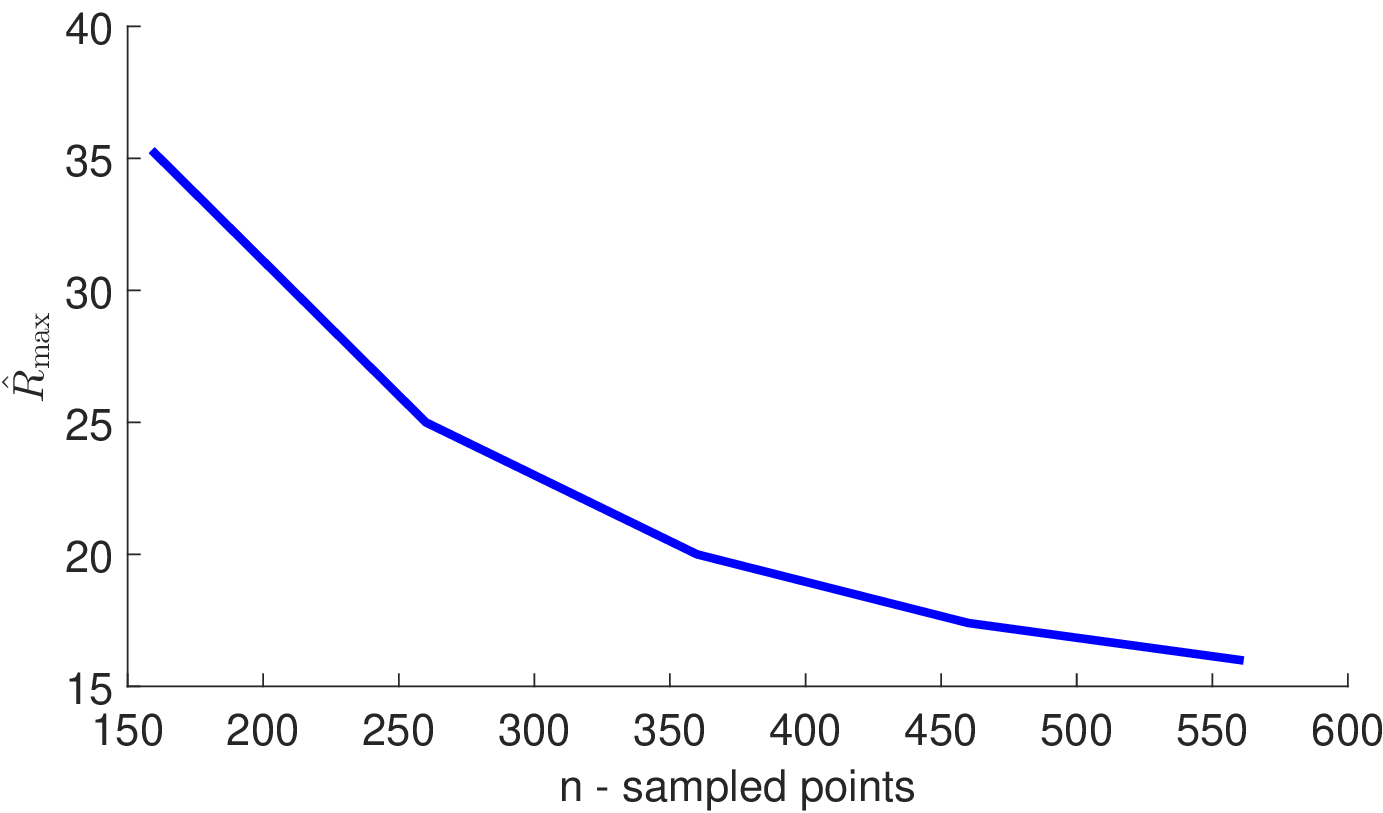} 
                        \label{fig:yalermax_a}}
        \end{subfloat}
        \caption{Simulated tests on real dataset Extended Yale Face
          Database B, with $N=2,414$ data dimension $D=114$ and
          $K=38$ clusters.}\label{fig:yale_a} 
\end{figure}

The PokerHand database contains $N=10^6$ data, belonging to $K=10$
clusters. Each datum is a $5$-card hand drawn from a deck of $52$
cards, with each card being described by its suit (spades, hearts,
diamonds, and clubs) and rank. Each cluster represents a valid Poker
hand. Fig.~\ref{fig:pokerhand} compares SkeVa-SC with SSSC on this
dataset after selecting $C = 10^{-2}$, $\mathbf{H} =
h^2(C,n)\mathbf{I}_D, \mathbf{H}_0 = ({h^2(C,n)}/{4})\mathbf{I}_D,
\mathbf{H}' = h^2(C,n')\mathbf{I}_D$, $n' = 800$, and $R_{\max} =
150$. Similar to the previous datasets, Alg.~\ref{alg} enjoys higher
accuracy than SSSC, while retaining low clustering time, corroborating
the fact that Alg.~\ref{alg} can handle very large datasets. However,
NMI is at suprisingly low levels, for both algorithms.
 
\begin{figure}[tb]
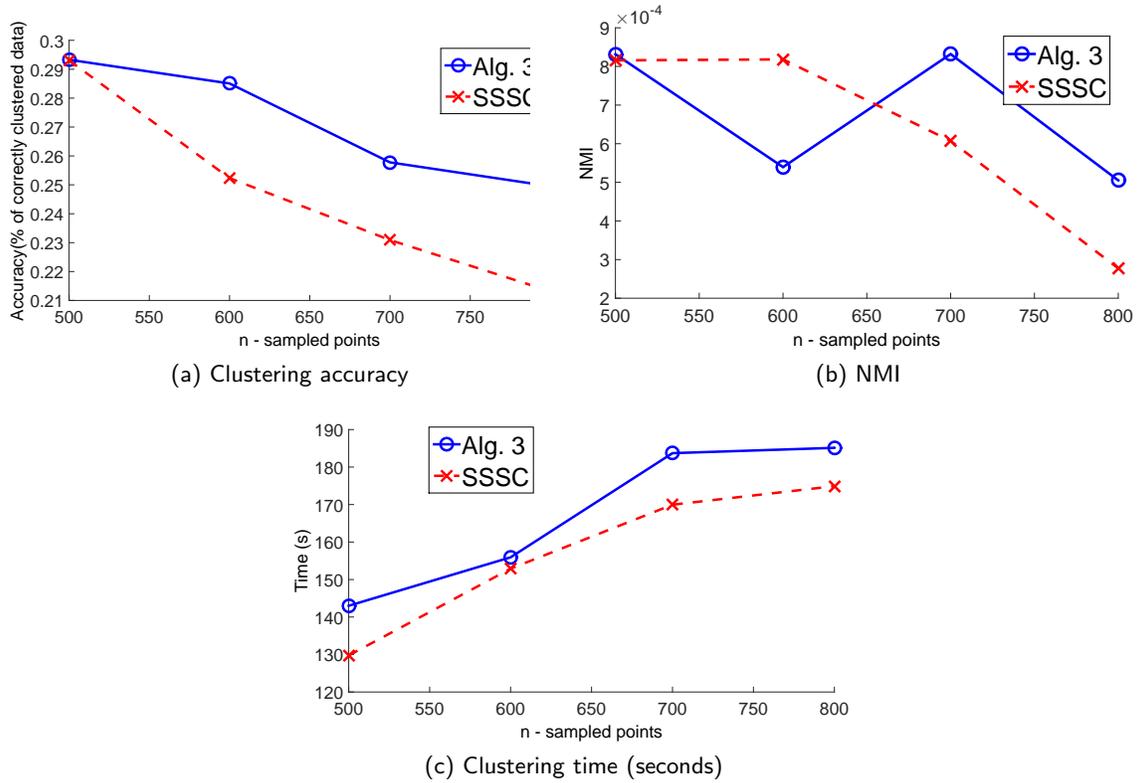

        \centering
        \begin{subfloat}[Clustering
          accuracy]{\includegraphics[width=0.45\columnwidth]{pokerhandacc.eps} 
                        \label{fig:pokerhandacc}}
        \end{subfloat}%
        \begin{subfloat}[NMI]{\includegraphics[width=0.45\columnwidth]{pokerhandnmi.eps} 
                        \label{fig:pokerhandnmi}}
        \end{subfloat}

        \begin{subfloat}[Clustering time
          (seconds)]{\includegraphics[width=0.45\columnwidth]{pokerhandtime.eps} 
                        \label{fig:pokerhandtime}}
        \end{subfloat}
        \caption{Simulated tests on real dataset Pokerhand, with $N=10^6$ data dimension $D=10$ and $K=10$ clusters.}\label{fig:pokerhand}
\end{figure}

\section{Conclusions and future work}\label{sec:conclusion}

The present paper introduced a novel iterative data-reduction scheme,
SkeVa-SC, that enables grouping of data drawn from a union of
subspaces based on a random sketching and validation approach for
fast, yet-accurate subspace clustering (SC). Part of the proposed
algorithm builds on the sparse SC algorithm, but it can also utilize
any other SC algorithm. Analytical bounds were derived on the number
of required iterations, and performance of the algorithm was
evaluated on synthetic and real datasets. Future research directions
will focus on the development of online SkeVa-SC, able to handle not
only big, but also fast-streaming data.

\begin{appendices}
\section{Integrated square error performance}\label{app:boundsISE} 

Using the definition of $d$, $\forall f,g\in\mathcal{X}$, we have
\begin{equation}
  \label{eq:L2dist}
  d(f,g) := \sqrt{d_{\text{ISE}}(f,g)} := \sqrt{\int\left(f(\bm{x}) -
    g(\bm{x})\right)^2d\bm{x}}\,.
\end{equation}
Based on \eqref{eq:L2dist}, the concavity of $\sqrt{\cdot}$, and Jensen's
inequality~\cite{cover2012elements}, an upper bound on
$\Expect[d(\cdot,\cdot)]$ is
\begin{align}
  \Expect[d(\cdot,\cdot)]  =
  \Expect \left[\sqrt{d_{\text{ISE}}(\cdot,\cdot)}\right] \leq
  \sqrt{\Expect \left[d_{\text{ISE}}(\cdot,\cdot)\right]}
   = \sqrt{\Expect
    \left[d^2(\cdot,\cdot)\right]}\label{eq:jensendelta}
\end{align}
and consequently, $\Expect^2[d(\cdot,\cdot)] \leq
\Expect[d^2(\cdot,\cdot)]$.

To ensure positivity, the number of required iterations of
\eqref{eq:basicRlowerbound} can be rewritten as
\begin{equation}
R_{\max} \ge
\frac{\log(\frac{1}{1-p})}{\log(\frac{1}{\Pr(\mathcal{B}_\delta)})}. 
\end{equation}
Thus, for a fixed $p$ the lower bound on the number of required iterations
grows as the probability of a ``bad'' event (as defined in
Def.~\ref{def:prbadevent}) increases.

The bad event probability of \eqref{eq:basicbadevent} can be lower-bounded using the extended Markov inequality for the quadratic function in $\mathbb{R}_{+}$, namely
\begin{equation}\label{eq:prbadgeexpect_extend}
\prob(\mathcal{B}_{\delta}) \ge 1 -
\frac{\Expect\left[d^2(\hat{f},f_0)\right]}{\delta_0^2}\,.   
\end{equation}
For $\Expect[d^2(\hat{f},f_0)]$ fixed, the larger
$\delta_0$ the larger the lower bound on the probability of
$\mathcal{B}_\delta$ becomes, and consequently the lower bound on
$R_{\max}$ increases.

Based on \eqref{eq:markovd0}, define $\delta_0 := d(f_0,\hat{f})$,
which is a random variable given its dependence on $\hat{f}$. By the
triangle inequality, $\delta_0$ can be bounded as
\begin{equation}
\label{eq:triangle_app}
|\delta' - d(f,\hat{f})|\le \delta_0 \le d(f,\hat{f}) + \delta'\,.
\end{equation}
The upper bound of \eqref{eq:triangle_app} will be used to provide a data-driven and relatively ``safe'' lower bound on the probability of $\mathcal{B}_\delta$, using the aforementioned observations.
While $d(f,\hat{f})$ is a quantity that depends on $\hat{f}$, and it
is thus random, it will be further bounded by deterministic
quantities, yielding a deterministic bound on $\delta_0$.

Distance $\delta':= d(f,f_0)$ can be expressed in closed form using
(As.~\ref{as:mixturemult}) and property \eqref{eq:gaussiantrickmult} as 
\begin{alignat}{2}
    \delta'^2 & = &&\; d^2(f,f_0) = \int\left(f(\bm{x}) -
      f_0(\bm{x})\right)^2d\bm{x}  
     =  \int f^2(\bm{x})d\bm{x} + \int f_0^2(\bm{x})d\bm{x} - 2\int
    f(\bm{x})f_0(\bm{x})d\bm{x} \notag\\ 
    &  = && \int\sum_{\ell}\sum_{\ell'}w_{\ell}w_{\ell'}
    \phi_{\mathbf{\Sigma}_{\ell}}(\bm{\mu}_{\ell}  
    - \bm{x})\phi_{\mathbf{\Sigma}_{\ell'}}(\bm{\mu}_{\ell'} - \bm{x}) 
    + \int\phi_{\mathbf{H}_0}(\bm{\mu}_{0} -
    \bm{x})\phi_{\mathbf{H}_0}(\bm{\mu}_{0} - \bm{x}) \notag\\ 
    & && - 2\int\sum_{\ell}w_l\phi_{\mathbf{\Sigma}_{\ell}}(\bm{\mu}_{\ell} -
    \bm{x})\phi_{\mathbf{H}_0}(\bm{\mu}_{0} - \bm{x}) \notag\\ 
    & = &&\; \mathbf{w}^{\top}\mathbf{\Omega}_0\mathbf{w} +
    \frac{1}{(4\pi)^{D/2}|\mathbf{H}_0|^{1/2}} 
    - 2\sum_{\ell}w_{\ell}\phi_{\mathbf{\Sigma}_{\ell} +
      \mathbf{H}_0}(\bm{\mu}_{\ell} -
    \bm{\mu_0}) \label{eq:deltaprime_mult}  
\end{alignat}
where $\mathbf{w} := [w_1,w_2,\ldots,w_L]^{\top}$ is formed by the mixing
coefficients of \eqref{eq:mixtureofgaussiansmult}, and $|\mathbf{H}|$
denotes the determinant of $\mathbf{H}$. Distances $d(\hat{f},f)$ and
$d(\hat{f},f_0)$ are random variables since they depend on $\hat{f}$, which
in turn depends on the randomly drawn data $\bm{x}_i$. Due to
(As.~\ref{as:mixturemult}), their expectation w.r.t.\ the true data pdf $f$
can be expressed in closed form. As data are drawn independently per
iteration $r$, expectations do not depend on $r$. As such, we have
  \begin{alignat}{2}
    \Expect[d^2(\hat{f},f_0)] & =\; &&
    \Expect[d_{\text{ISE}}(\hat{f},f_0)] = 
    \Expect\left[\int\left(\hat{f}(\bm{x}) -
        f_0(\bm{x})\right)^2d\bm{x})\right] \hspace{100pt}  \notag\\
    & = &&  \int f_0^2(\bm{x})d\bm{x} + \Expect\left[\int
      \hat{f}^2(\bm{x})d\bm{x}\right] 
     - 2\Expect\left[\int
      f_0(\bm{x})\hat{f}(\bm{x})d\bm{x}\right]  \notag\\
    & = &&  \frac{1}{(4\pi)^{D/2}|\mathbf{H}_0|^{1/2}} +
    \int\Expect\left[\hat{f}^2(\bm{x})d\bm{x}\right]  
     - 2\Expect\left[\int f_0(\bm{x})\hat{f}(\bm{x})d\bm{x}\right]
    \notag \\   
    & = && \frac{1}{(4\pi)^{D/2}|\mathbf{H}_0|^{1/2}} +
    \frac{1}{n}\frac{1}{(4\pi)^{D/2}|\mathbf{H}|^{1/2}} 
     + \left(1 -
      \frac{1}{n}\right)\mathbf{w}^{\top}\mathbf{\Omega}_2\mathbf{w}
     - 2\sum_{\ell}w_{\ell}\phi_{\mathbf{H}+\mathbf{H}_0 + 
      \mathbf{\Sigma}_{\ell}}(\bm{\mu}_{\ell} -
    \bm{\mu}_0) \label{eq:delta0_mult}   
\end{alignat}
and
\begin{alignat}{1}
  \Expect[d^2(\hat{f},f)]  = \; \Expect[d_{\text{ISE}}(\hat{f},f)] 
   = \; \frac{1}{n(4\pi)^{D/2}|\mathbf{H}|^{1/2}} 
   +
  \mathbf{w}^{\top}\left(\left(1-\frac{1}{n}\right)\mathbf{\Omega}_2 -
    2\mathbf{\Omega}_1 +
    \mathbf{\Omega}_0\right)\mathbf{w} \label{eq:delta_mult}
\end{alignat}
with 
\begin{alignat}{2}
    \int\Expect[\hat{f}^2(\bm{x})]d\bm{x} & = && \;
    \int\frac{1}{n^2}\sum_{i}\sum_{j}\Expect\left[ 
      \phi_{\mathbf{H}}(\bm{x}-\bm{x}_i)
      \phi_{\mathbf{H}}(\bm{x}-\bm{x}_j)\right]d\bm{x} \notag\\ 
    & = &&\; \int\left\{\frac{1}{n^2}\sum_{i=j}\Expect\left
      [\phi_{\mathbf{H}}^2(\bm{x}-\bm{x}_i)\right] \right.  
     \left. + \frac{1}{n^2}\sum_{i\neq j}\Expect
    \left[\phi_{\mathbf{H}}(\bm{x}-\bm{x}_i)
      \phi_{\mathbf{H}}(\bm{x}-\bm{x}_j)\right]\right\} d\bm{x}   \notag\\ 
    & = &&\; \int\left\{\frac{1}{n}\Expect\left[
      \phi_{\mathbf{H}}^2(\bm{x}-\bm{x}_i)\right] \right. 
     \left. + \frac{n^2 - n}{n^2}\Expect^2\left[
      \phi_{\mathbf{H}}(\bm{x}-\bm{x}_i)\right]\right\}d\bm{x}  \notag\\ 
    & = &&\; \frac{1}{n}\frac{1}{\left(4\pi\right)^{D/2}|\mathbf{H}|^{1/2}} +
    \left(1-\frac{1}{n}\right)\mathbf{w}^{\top}
    \mathbf{\Omega}_2\mathbf{w} \label{eq:expectfhat2}  
\end{alignat}
where the third equality is due to the fact that $\bm{x}_i$'s are
independently drawn from $f$. Moreover,
\begin{align}
    \Expect  \left[\int f_0(\bm{x})\hat{f}(\bm{x})d\bm{x}\right]
     & = \frac{1}{n}\Expect\left[\sum_{i}\int\phi_{\mathbf{H}_0}(\bm{x}-\bm{\mu}_0)
      \phi_{\mathbf{H}}(\bm{x}-\bm{x}_i)d\bm{x}\right]   
     = \frac{1}{n}\Expect\left[\sum_{i}\phi_{\mathbf{H}+\mathbf{H}_0}(\bm{x}_{i}
      -\bm{\mu}_0)\right] \notag\\
    & = \Expect\left[\phi_{\mathbf{H}+\mathbf{H}_0}(\bm{x}
      -\bm{\mu}_0)\right]   
     = \int\sum_{\ell}w_{\ell}\phi_{\mathbf{\Sigma}_{\ell}}(\bm{x}-
    \bm{\mu}_{\ell})\phi_{\mathbf{H}+\mathbf{H}_0}(\bm{x}-\bm{\mu}_0)d\bm{x} \notag\\
    & = \sum_{\ell}w_{\ell}\phi_{\mathbf{H} + \mathbf{H}_0 + \mathbf{\Sigma}_{\ell}}
    (\bm{\mu}_{\ell}-\bm{\mu}_0)\,. \label{eq:expectfhatf0}  
\end{align}

Interestingly, the probability of $d(f,\hat{f})$ being far from its
ensemble average $\Expect[d(f,\hat{f})]$ can be bounded via the general
inequality~\cite{devroye1991exponential}
\begin{align}
   \Pr \left(\left|\|f - \hat{f}\|_p - \Expect \left[\|f -
        \hat{f}\|_p \right]\right| \ge t\right) 
    \le 2\exp\left(-\frac{nt^2h^{2 -
        2/p}}{2\|K_{\mathbf{I}}(\bm{x})\|_p^2}\right) \label{eq:ebounddelta} 
\end{align} 
where $\|\cdot\|_p$ denotes the $L_p$-norm for $p\ge 1$, i.e.,
$\|f\|_p := \left(\int |f(x)|^pdx\right)^{1/p}$. For the
Gaussian kernel with covariance matrix $\mathbf{H} := 
h^2\mathbf{I}_D$ and $p=2$, the norm $\|K_{\mathbf{I}}(\bm{x})\|_2^2$ becomes
\begin{align}
\|K_{\mathbf{I}}(\bm{x})\|_2^2  =  \|\phi_{\mathbf{I}}(\bm{x})\|_2^2
 = \int \phi_{\mathbf{I}}^2(\bm{x}) d\bm{x} =
\phi_{2\mathbf{I}}(\bm{0}) = \frac{1}{(4\pi)^{D/2}}\,. \label{eq:Knorm} 
\end{align}
Consequently, \eqref{eq:ebounddelta} reduces to
\begin{align}
    \Pr\left(\left|d(f,\hat{f}) -
        \Expect\left[d(f,\hat{f})\right]\right| \ge t\right) 
     \le
    2\exp\left(-\frac{nt^2h\left(4\pi\right)^{D/2}}{2}\right)\,.
    \label{eq:ebounddelta2}     
\end{align}
Letting $q := \Pr(|d(f,\hat{f}) - \Expect[d(f,\hat{f})]| \ge t)$,
\eqref{eq:ebounddelta2} yields 
\begin{align}
\label{eq:qbound}
q \le 2\exp\left(-\frac{nt^2h\left(4\pi\right)^{D/2}}{2}\right)
\end{align}
 and solving \eqref{eq:qbound} w.r.t.\ $t$, we arrive at
 \begin{align}
 t \le \sqrt{-\frac{2\log(q/2)}{nh\left(4\pi\right)^{D/2}}}\,.
 \end{align}
 Since $1-q = \Pr(|d(f,\hat{f}) - \Expect[d(f,\hat{f})]| < t)$, the
 distance $d(f,\hat{f})$ can be upper bounded with
 probability $1-q$ as
 \begin{align}
   d(f,\hat{f})  <
   \sqrt{-\frac{2\log(q/2)}{nh\left(4\pi\right)^{D/2}}} + 
   \Expect[d(f,\hat{f})] 
    \le \sqrt{-\frac{2\log(q/2)}{nh\left(4\pi\right)^{D/2}}} + 
   \sqrt{\Expect[d^2(f,\hat{f})]} \label{eq:deltabound}
 \end{align}
 where the second inequality follows readily from
 \eqref{eq:jensendelta}.


 An upper bound with probability $1-q$ to obtain the value of
 $\delta_0$ can be now derived using \eqref{eq:triangle_app} and
 \eqref{eq:deltabound}:
 \begin{align}
   \delta_0 \le d(f,\hat{f}) + \delta' 
   \le \sqrt{-\frac{2\log(q/2)}{nh\left(4\pi\right)^{D/2}}} + 
   \sqrt{\Expect[d^2(f,\hat{f})]} + \delta' =:
     \theta. \label{eq:delta_0bound} 
 \end{align}
Using $\theta$ in place of $\delta_0$ will yield an increased lower
bound to $R_{\max}$, as it will increase the value of $1 -
\Expect[d^2(\hat{f},f_0)]/\delta_0^2$, in \eqref{eq:prbadgeexpect_extend}; thus, with probability $1-q$
\begin{align}
	 1 - \frac{\Expect\left[d^2(\hat{f},f_0)\right]}
	{\left(\delta' + \sqrt{-\frac{2\log(q/2)}{nh\left(4\pi\right)^{D/2}}}  
		+ \sqrt{\Expect[d^2(f,\hat{f})]}\right)^2}  
	\ge 1 - \frac{\Expect\left[d^2(\hat{f},f_0)\right]}
	{\delta_0^2}\,. \label{eq:lowerboundlog} 
\end{align}



Finally, using \eqref{eq:lowerboundlog}, an overestimate $\hat{\varrho}$,
with probability $1-q$, of the lower bound $\varrho$ in
\eqref{eq:theoremRlowerbound} is
\begin{equation}
  \label{eq:lowerboundRapp}
  \hat{\varrho} := \frac{\log(1-p)}{\log\left(1 -
      \frac{\Expect\left[d^2(\hat{f},f_0)\right]}
      {\left(\delta' + \sqrt{-\frac{2\log(q/2)}{nh\left(4\pi\right)^{D/2}}}  
        + \sqrt{\Expect[d^2(f,\hat{f})]}\right)^2}\right)}.
\end{equation}
Upon 
$\theta_1 :=  (-{2\log(q/2)}/{(nh(4\pi)^{D/2})})^{1/2} +
(\Expect[d^2(f,\hat{f})])^{1/2}$, and $\theta_2 := \delta'$, the claim of
Thm.~\ref{th:Rbound}.\ref{thm:2nd.part} is established.
\end{appendices}



\bibliographystyle{IEEEtran}
\bibliography{./bib/PAMI_bib}

\end{document}